\documentclass{article}

\usepackage{microtype}
\usepackage{graphicx}
\usepackage{subfigure}
\usepackage{booktabs} 
\usepackage{cancel}
\usepackage{hyperref}

\usepackage[accepted]{icml2024}
\usepackage{amsmath}
\usepackage{amssymb}
\usepackage{mathtools}
\usepackage{amsthm}

\usepackage[capitalize,noabbrev]{cleveref}

\theoremstyle{plain}
\newtheorem{theorem}{Theorem}[section]

\newtheorem{lemma}[theorem]{Lemma}

\theoremstyle{definition}
\newtheorem{definition}[theorem]{Definition}

\usepackage[textsize=tiny]{todonotes}

\usepackage{amsmath,amsfonts,bm}

\def\eqref#1{equation~\ref{#1}}

\def\1{\bm{1}}

\DeclareMathAlphabet{\mathsfit}{\encodingdefault}{\sfdefault}{m}{sl}
\SetMathAlphabet{\mathsfit}{bold}{\encodingdefault}{\sfdefault}{bx}{n}

\newcommand{\E}{\mathbb{E}}

\DeclareMathOperator*{\argmin}{arg\,min}

\usepackage{hyperref}
\usepackage{url}
\usepackage{amsmath,amsthm,amssymb,amsfonts,bbm,bbding,bm}
\usepackage{enumitem}
\usepackage{etoolbox}
\usepackage{graphicx}
\usepackage{hyperref}
\usepackage{xcolor}
\newcommand\EnumPrefix{}
\usepackage{multirow}
\usepackage{mathtools}
\newlist{senenum}{enumerate}{10}
\setlist[senenum]{label=\arabic*.,ref=\EnumPrefix,leftmargin=*}
\AtBeginEnvironment{lemmalist}{\renewcommand\EnumPrefix{\thelemmalist.\arabic*}}
\AtBeginEnvironment{assumptionlist}{\renewcommand\EnumPrefix{\theassumptionlist.\arabic*}}
\newtheorem{lemmalist}[theorem]{Lemma}

\newcommand{\mcalD}{\mathcal{D}}

\newcommand{\luo}{\color{black}}

\numberwithin{equation}{section}
\usepackage{wrapfig}
\usepackage{algorithmic,algorithm}
\icmltitlerunning{Uniformly Stable Algorithms for Adversarial Training and Beyond}

\begin{document}

\twocolumn[
\icmltitle{Uniformly Stable Algorithms for Adversarial Training and Beyond}
\icmlsetsymbol{equal}{*}

\begin{icmlauthorlist}
\icmlauthor{Jiancong Xiao}{equal,upenn}
\icmlauthor{Jiawei Zhang}{equal,mit}
\icmlauthor{Zhi-Quan Luo}{cuhksz}
\icmlauthor{Asuman Ozdaglar}{mit}
\end{icmlauthorlist}

\icmlaffiliation{upenn}{University of Pennsylvania, PA, USA;}
\icmlaffiliation{mit}{Massachusetts Institute of Technology, MA, USA;}
\icmlaffiliation{cuhksz}{The Chinese University of Hong Kong, Shenzhen, China}

\icmlcorrespondingauthor{Jiancong Xiao}{jcxiao@upenn.edu}
\icmlcorrespondingauthor{Jiawei Zhang}{jwzhang@mit.edu}

\icmlkeywords{Machine Learning, ICML}

\vskip 0.3in
]
\printAffiliationsAndNotice{\icmlEqualContribution}

\begin{abstract}
In adversarial machine learning, neural networks suffer from a significant issue known as robust overfitting, where the robust test accuracy decreases over epochs \citep{rice2020overfitting}. Recent research conducted by \citet{xing2021on,xiao2022stability} has focused on studying the uniform stability of adversarial training. Their investigations revealed that SGD-based adversarial training fails to exhibit uniform stability, and the derived stability bounds align with the observed phenomenon of robust overfitting in experiments. This motivates us to develop uniformly stable algorithms specifically tailored for adversarial training. To this aim, we introduce Moreau envelope-$\mathcal{A}$, a variant of the Moreau Envelope-type algorithm. We employ a Moreau envelope function to reframe the original problem as a min-min problem, separating the non-strong convexity and non-smoothness of the adversarial loss. Then, this approach alternates between solving the inner and outer minimization problems to achieve uniform stability without incurring additional computational overhead. In practical scenarios, we show the efficacy of ME-$\mathcal{A}$ in mitigating the issue of robust overfitting. Beyond its application in adversarial training, this represents a fundamental result in uniform stability analysis, as ME-$\mathcal{A}$ is the first algorithm to exhibit uniform stability for weakly-convex, non-smooth problems. 
\end{abstract}

\section{Introduction}
One of the interesting ability of deep neural networks (DNNs) \citep{krizhevsky2012imagenet} is that they rarely suffered from overfitting issues \citep{zhang2021understanding}. However, in the setting of adversarial training, this ability disappears, and overfitting becomes one of the most critical issues. Specifically, in a regular setting of SGD-based adversarial training shown in Figure \ref{fig:intro}, the robust test accuracy (orange line) starts to decrease after a particular epoch, while the robust training accuracy (blue line) continues to increase. This phenomenon is referred to as robust overfitting \citep{rice2020overfitting}. It can be observed in experiments on common datasets such as SVHN, CIFAR-10/100.
\begin{figure}[htbp]
	\vspace{-0.1in}
	\centering
	\includegraphics[width=0.8\linewidth]{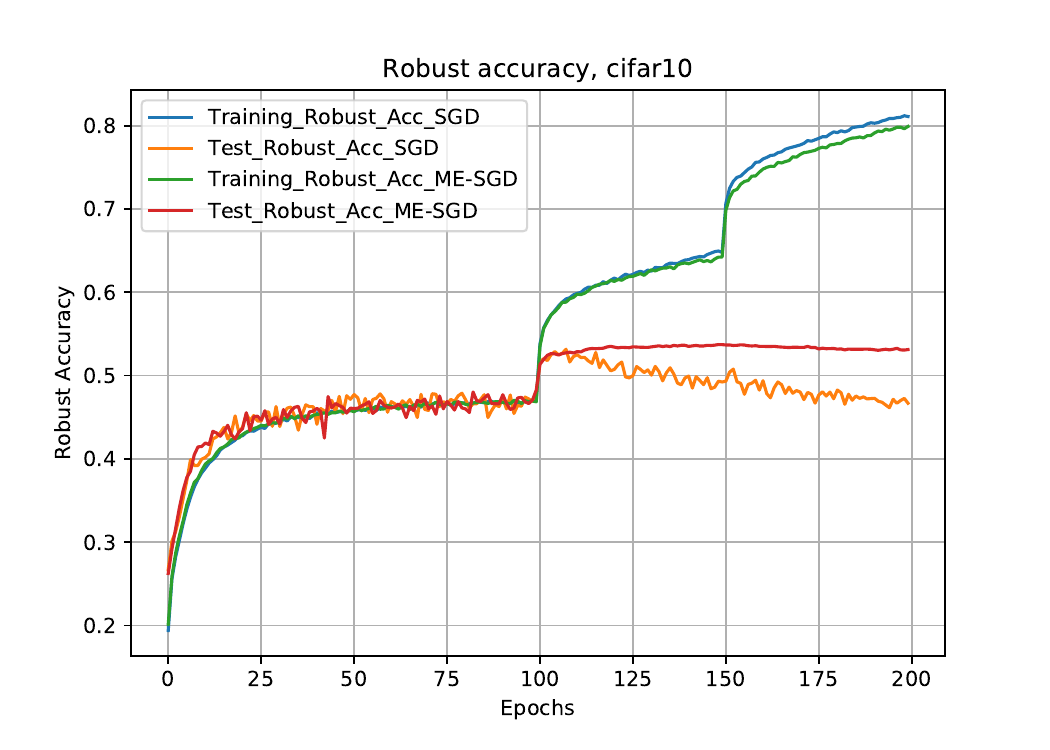}
	\caption{\small{Demonstration of robust overfitting (Orange Line) and mitigating robust overfitting by Moreau envelope-$\mathcal{A}$ (Red Line).}}
	\label{fig:intro}
\end{figure}

\begin{table*}[htbp]
	\centering
 \caption{Comparison of uniform stability bounds of adversarial training. ME-$\mathcal{A}$ reduces the additional term in $\mathcal{O}(T^q\epsilon)$, which is a possible reason for robust overfitting. It also mitigates the robust overfitting issue in practice.}
\begin{tabular}{cccccc}
		\toprule
	& Assumption (Std-loss)&Training & Algorithm & Stability Bound  & No Overfitting\\
	\midrule
 	\citet{hardt2016train} &Nonconvex &Standard Training&SGD & $\mathcal{O}(T^q/n)$ &\Checkmark\\
	\citet{xiao2022stability}&Nonconvex  &Adversarial Training&SGD & $\mathcal{O}(T^q\epsilon+T^q/n)$ &\XSolidBrush\\
	\textbf{Ours}&Nonconvex &Adversarial Training & \textbf{ME-$\mathcal{A}$} & $\mathcal{O}(T^q/n)$  & \Checkmark\\
	\bottomrule
\end{tabular}\label{table:intro}
\end{table*}

Recent research has utilized uniform stability, a generalization measure in learning theory, to investigate this phenomenon \citep{xing2021on,xiao2022stability}. They have suggested that the non-smoothness of the adversarial loss may contribute to the issue of robust overfitting. Informally, uniform stability is the gap between the output parameters $w$ of running an algorithm on two datasets $S$ and $S'$ differ in at most one sample, denoted as $\|w(S)-w(S')\|$.  In uniform stability analysis, assuming the standard training loss is non-convex and smooth, a well-known result given by \cite{hardt2016train} is that applying stochastic gradient descent (SGD) to the standard loss yields uniform stability in $\mathcal{O}(T^q/n)$, where $T$ represents the number of iterations, $n$ is the number of samples and $0<q<1$. However, the adversarial loss is non-smooth, even if we assume the standard loss is smooth. Consequently, the uniform stability bounds include an additional term in $\mathcal{O}(T^q\epsilon)$ \cite{xiao2022stability}, where $\epsilon$ is the attack intensity. The bound suggests that the robust test error increases as $T$ grows, even when we have an infinite number of training samples (\emph{n $\rightarrow$ $\infty$}). Therefore, the derived uniform stability bounds align with the observations made in practical adversarial training.

This observation motivates us to develop uniformly stable algorithms for adversarial training in order to mitigate the issue of robust overfitting. Additionally, we define the adversarial loss as $h(w;z)=\max_{\|z-z'\|\leq\epsilon}g(w;z)$, where $z$ represents the sample, $w$ is the training parameter, and $g$ denotes the corresponding standard training loss. It has been proven that the additional term arises from the non-smoothness caused by the $\max$ operation in the adversarial loss \cite{xiao2022stability}. Consequently, our objective is to design uniformly stable algorithms specifically tailored for non-smooth optimization problems and then apply them to the adversarial training problem.

However, attaining uniform stability with non-smooth loss functions presents significant challenges. Common strategies in non-smooth optimization involve smoothing the loss function (\emph{e.g.,} Moreau-Yosida smoothing (MYS) \citep{nesterov2005smooth}). This approach is notably computationally inefficient when aiming for uniform stability, as proved by \citet{bassily2019private}. In deep learning, this drawback is further magnified and such methods are computationally intractable. Additionally, a notable study by \citet{bassily2020stability} pointed out that stochastic gradient descent (SGD), a widely used optimization technique, does not ensure uniform stability for convex non-smooth problems.

Our approach involves utilizing the Moreau envelope function \citep{moreau1965proximite}, which is a classical tool employed in methods such as the proximal point method (PPM) and MYS, in a different way. Our approach is to separate the non-strong convexity and non-smoothness of the original convex non-smooth problem. Given training set $S$, the original problem (P.1) is then reformulated into the min-min problem (P.2).
\begin{equation}
\begin{aligned}
   & \min_w \mathbb{E}_S h(w;z)\quad(\text{P.1})\\
   \Leftrightarrow \  &\min_u\min_w \mathbb{E}_S \big[h(w;z)+\frac{p}{2}\|w-u\|^2\big]\quad(\text{P.2})
\end{aligned}
\end{equation}
Firstly, it is important to note that P.2 is equivalent to P.1 in terms of global solutions. Secondly, the inner problem exhibits strong convexity and non-smoothness, while the outer problem is convex and smooth. Both of the two problems have at least one advantage for algorithm design. 

Moreover, this analysis extends beyond convex problems to include non-convex ones as well. Building on the assumptions in \citet{hardt2016train,xiao2022stability} that the standard loss function $g(w;z)$ is non-convex and smooth, although the adversarial loss is both non-convex and non-smooth, it can still be demonstrated to be weakly convex. Therefore, ME-$\mathcal{A}$ can be applied in this setting.

The main results of this work are in three aspects.

\begin{senenum}
\item \textbf{Algorithms for Adversarial Training.} Let $\mathcal{A}$ be a first-order algorithm used to solve the original problem (P.1), such as stochastic gradient descent (SGD) or batch gradient descent (BGD). We introduce Moreau envelope-$\mathcal{A}$ (ME-$\mathcal{A}$), which alternates the application of $\mathcal{A}$ to the inner problem and GD to the outer problem of P.2. We prove that ME-$\mathcal{A}$ achieves uniform stability for both the inner and outer problems, thereby achieving uniform stability for the entire problem without incurring additional computational overhead. The comparison of SGD and ME-$\mathcal{A}$ is provided in Table \ref{table:intro}. ME-$\mathcal{A}$ improves over SGD in terms of uniform stability by reducing the term $\mathcal{O}(T^q\epsilon)$. In Figure \ref{fig:intro} (red line), we demonstrate that ME-$\mathcal{A}$ effectively mitigates robust overfitting in practical scenarios.

\item \textbf{Understanding Adversarial Training.} In the previous studies of adversarial training, robust overfitting and sample complexity are usually considered to be related: DNNs tend to overfit adversarial examples, necessitating more samples to circumvent this issue \citep{schmidt2018adversarially,rice2020overfitting}. This paper provides us with further insights that robust generalization can be decomposed additively by robust overfitting and sample complexity, i.e., 
\begin{equation*}
\begin{aligned}
    &\text{Robust Generalization}\\\leq& \underbrace{\text{Robust Overfitting}}_{\text{red line - orange line}\ \approx\ \mathcal{O}(T^q\epsilon)} +\underbrace{\text{Sample Complexity}}_{\text{blue line - red line}\ \approx\  \mathcal{O}(T^q/n)}.
\end{aligned}
\end{equation*}
\emph{Robust Overfitting.} By employing ME-$\mathcal{A}$, the robust overfitting issue (in $\mathcal{O}(T^q\epsilon)$) is mitigated. DNNs fit the adversarial examples well, yet achieving the performance ceiling (red line in Figure \ref{fig:intro}), within the constraints of the existing dataset size $n$. A widely used algorithm, stochastic weight averaging, plays a similar role as ME-$\mathcal{A}$ in adversarial training.

\emph{Sample Complexity.} Considering the performance ceiling established at $\mathcal{O}(T^q/n)$, an increase in data volume is essential for enhancing this upper limit. Recent studies, such as those by \citet{carmon2019unlabeled} on pseudo-labeled data and \citet{rebuffi2021fixing} on generated data, have demonstrated their effectiveness in improving robust generalization. These findings lend support to our theoretical framework.

\item \textbf{Beyond Adversarial Training.} While our primary emphasis lies in adversarial training, we present a fundamental result in uniform stability analysis. ME-$\mathcal{A}$ is the first uniformly stable algorithm for weakly-convex, non-smooth problems, which is not accomplished by existing algorithms such as PPM and MYS. A comprehensive comparison is given in Sec. \ref{further}.
\end{senenum}

\section{Related Work} 
\label{rel}

\paragraph{Uniform Stability.} The concept of stability can be traced back to the work of \citep{rogers1978finite}. In the context of statistical learning problems, it has been well developed through the analysis of algorithm-based generalization bounds \citep{bousquet2002stability}.
\vspace{-0.1in}
\paragraph{Smooth Cases.} In smooth settings, \citet{hardt2016train} established strong bounds of stability. They demonstrated that several variants of stochastic gradient descent (SGD) can simultaneously achieve uniform stability bounds in $\mathcal{O}(L^2T\alpha/n)$ in convex settings and $\mathcal{O}(L^2T^q/n)$ in non-convex settings, where $\alpha$ is the step size. This approach has been used in several studies to derive new generalization properties of SGD \citep{feldman2018generalization,feldman2019high}. The work of \citep{chen2018stability} investigated the optimal trade-off between stability and convergence.
\vspace{-0.1in}
\paragraph{Proximal Methods.} ME-$\mathcal{A}$ looks similar to proximal point method (PPM) since they both use the Moreau envelope function. In smooth case, \citet{yuan2023sharper} have provided a thorough analysis of uniform stability for PPM. Additionally, \citet{hardt2016train} demonstrated that incorporating a proximal step following SGD steps does not degrade the uniform stability bound. In Sec. \ref{further}, we provide a more detailed discuss about the difference between ME-$\mathcal{A}$ and PPM.
\vspace{-0.1in}
\paragraph{Non-Smooth Cases.} In adversarial training, \citet{liu2020loss} proved that non-smoothness is an important issue, leading to bad robust accuracy. \citet{bassily2020stability} investigated the stability of several variants of stochastic gradient descent (SGD) on non-smooth loss functions. They demonstrated that the generalization bound contains an additional terms. Subsequent studies showed that some variants of SGD, such as pairwise-SGD \citep{yang2021stability} and Markov chain-SGD \citep{wang2022stability}, also possess this term. We list these work in Appendix C. \citet{kanai2023relationship} introduces the use of EntropySGD, a technique that applies SGLD to a smooth surrogate loss to enhance uniform stability. The study of \citep{lei2022stability} also investigated the topic of stability in non-convex non-smooth problems. They introduced a novel stability measure known as stability in gradient, which assesses the stability of non-convex problems. Remarkably, they also employed the Moreau envelope function, but for the definition of stability in gradient for non-differentiable function. 

\paragraph{Uniform Convergence Analysis.} Besides algorithmic generalization analysis, uniform convergence represents a different approach to generalization analysis in traditional learning theory. It offers generalization bounds for the function class  with high probability, which is algorithm-independent. Uniform convergence analysis includes VC-dimension, Rademacher complexity, and Pac-Bayes analysis. Research by \citet{cullina2018pac,attias2022improved,attias2022characterization} had established adversarial generalization bounds utilizing VC-dimension. For example, in finite adversarial examples cases, \citet{attias2022improved} provided the sample complexity of generalization gap respect to VC($\mathcal{H}$). Following that, \citet{montasser2019vc} have proved that VC classes are robustly PAC-learnable only improperly, with respect to any arbitrary perturbation set, possibly of infinite size. Their approach relies on sample compression arguments whereas uniform convergence does not hold. Regarding Rademacher complexity, robust generalization can be bounded by adversarial Rademacher complexity \citep{khim2018adversarial,yin2019rademacher} for linear classifier. It is extended to two-layers neural networks \citep{awasthi2020adversarial}, FGSM attacks \citep{gao2021theoretical}, and deep neural networks\citep{xiao2022adversarial,mustafa2022generalization}. Pac-bayes analysis is another approach to provide norm based control for generalization. \citet{farnia2018generalizable} and \citet{xiao2023pac} Pac-bayesian bound for adversarial generalization. Since these bounds are algorithm-independent, they cannot distinguish the generalization performance of different algorithms.
\section{Preliminaries of Stability Analysis} \label{updates:sec}

Let ${\cal D}$ be an unknown distribution in the sample
space $\mathcal{Z}$. 
Our goal is to find a model $w$ with small population risk, defined as:
\[
R_{\cal D}(w) = \E_{z\sim{\cal D}} h(w,z),
\]
where $h(\cdot,\cdot)$ is the loss function which is possibly nonsmooth. 
Since we cannot get access to the objective $R_{\mcalD}(w)$ directly due to the unknown distribution $\mcalD$, we instead minimize the empirical risk built on a training dataset. 
Let $S=\{z_1,\dots,z_n\}\sim {\cal D} ^n$ be an sample dataset drawn
i.i.d. from ${\cal D}.$
The empirical risk function is defined as:
\[
R_S(w)=\frac1n\sum_{i=1}^n h(w,z_i).
\]
 Let $\bar w$ be the optimal solution of $R_S(w)$. Then, for the algorithm output $\hat w=A(S)$, we define the expected generalization gap as 
\begin{equation}\label{gengap}
\begin{aligned}
\mathcal{E}_{gen} =\mathbb{E}_{S\sim \mathcal{D}^n,A} [R_{\cal D}(A(S))- R_{S}(A(S))].
    \end{aligned}
\end{equation}
 We define the the expected optimization gap as 
 \begin{equation}\label{optgap}
\begin{aligned}
\mathcal{E}_{opt} =\mathbb{E}_{S\sim \mathcal{D}^n,A} [R_{S}(A(S))- R_{S}(\bar w)].
    \end{aligned}
\end{equation}
To bound the generalization gap of a model $\hat w=A(S)$ trained by a randomized algorithm $A$, we employ the
following notion of \emph{uniform stability}.

\begin{definition}
A randomized algorithm $A$ is $\varepsilon$-\emph{uniformly stable} if
for all data sets $S,S'\in \mathcal{Z}^n$ such that $S$ and $S'$ differ in at most one
example, we have
\begin{equation}\label{eq:stab}
\sup_{z} \E_{A} \left[ h(A(S); z) - h(A(S'); z) \right] \le \varepsilon\,.
\end{equation}
\end{definition}
The following theorem shows that expected generalization gap can be attained from uniform stability.

\begin{theorem}[Generalization in expectation \citep{hardt2016train}]
\label{thm:stab2gen}
Let $A$ be $\varepsilon$-uniformly stable. Then, the expected generalization gap satisfies
\[
|\mathcal{E}_{gen}|=\left| \E_{S,A}[R_{\cal D}[A(S)] - R_S[A(S)]]\right| \le \varepsilon\,.
\]
\end{theorem}

\paragraph{Hypothesis Class.} As proved in \citep{xing2021on,xiao2022stability}, $h(w;z)$ is non-smooth even even though we assume that its standard counterpart is smooth. Therefore, we focus on non-smooth loss minimization. This class is denoted by $\mathcal{H}$ and is defined as follows:
\begin{equation*}
\label{H}
\begin{aligned}
    \mathcal{H} =\{h: W\times\mathcal{Z}\rightarrow \mathbb{R} \mid L\text{-Lipschitz in }w, |W|=D_W\}.
\end{aligned}
\end{equation*}
In this paper, we explore both convex and non-convex settings. When the standard loss is convex, the adversarial loss is also convex. In cases where the standard loss is non-convex, the adversarial loss can still be demonstrated to be weakly convex, owing to the smoothness of the standard loss.

In experiments, we consider the following two losses for adversarial training.
\paragraph{Adversarial Loss.} Let the loss function $h(w;z)=\max_{\|x-x'\|\leq \epsilon} \ell(f_w(x'),y),$ where $\epsilon$ is the perturbation intensity. Here $f_w(\cdot)$ is a neural network parameterized by $w$, and $z=(x,y)$ is the input-label pair. If the neural networks are defined in a compact domain, \emph{i.e.,}
$\|x\|\leq B
, \forall x\in \mathcal{X}$, the loss function $h(w;z)$ is $L$-Lipschitz. It is shown that adversarial loss is $L$-Lipschitz given the standard loss is $L$-Lipschitz, and it is non-smooth even if the standard loss is smooth \citep{xiao2022stability}.

\paragraph{TRADES Loss.} Let the loss function be $$h(w;z)=\ell(f_w(x'),y)+\beta\max_{\|x-x'\|\leq \epsilon}\ell(f_w(x'),f_w(x)),$$ where $\beta$ is a hyperparameter \citep{zhang2019theoretically}. Similar to the adversarial loss, the inner maximization problem induces the non-smoothness of the TRADES loss.

\section{Moreau Envelope-$\mathcal{A}$}
\label{s5}
In this section, we introduce the algorithms, Moreau Envelope-$\mathcal{A}$\footnote{The initial version of the algorithm is refer to as Smoothed-SGDmax \cite{xiao2022smoothed}.}, to achieve $\mathcal{O}(T^q/n)$-uniform stability for non-smooth loss minimization. Although our primary findings are framed within non-convex settings, we begin our discussion in convex settings to streamline the theoretical exposition.
\subsection{Equivalent Problem}
We use the Moreau envelope function to define the surrogate loss. Let 
{\small\begin{eqnarray}
K(w,u;z)=h(w;z)+\frac{p}{2}\|w-u\|^2.
\end{eqnarray}}
We can choose $p>0$ to insure that $K(w,u;z)$ is strongly convex with respect to $w$. We define the Moreau envelope function of the empirical loss:
{\small\begin{eqnarray}
 M(u;S)\nonumber&=&\min_{w\in W}K(w,u;S)=\min_{w\in W}\frac{1}{n}\sum_{z\in S}K(w,u;z),\label{eq:M}\\
w(u;S)&=& \arg\min_{w\in W}K(w,u;S).
\end{eqnarray}}
Employing the Moreau envelope function to the empirical loss (rather than the loss $h(w;z)$ as in MYS) is an important steps in our approach. We defer the comparison of ME-$\mathcal{A}$ and MYS to Section \ref{further} to show why it is important. The following Lemma holds for the Moreau envelope function $M(u;S)$.

\begin{lemmalist}[Equivalent Problem]
\label{l0}
Assume that $h\in\mathcal{H}$. Let $p>0$. Then, $\min_u M(u;S)$ has the same global solution set as $\min_w R_S(w)$.
\end{lemmalist}
As $\min_u M(u;S)$ has the same global solutions as $\min_wR_S(w)$, the original problem is equivalent to the problem of minimizing $M(u;S)$, \emph{i.e.,}
\begin{eqnarray}
&\min_w R_S(w)\\\Leftrightarrow & \min_u\min_{w}\frac{1}{n}\sum_{z\in S}K(w,u;z).\label{eq:Mloss}
\end{eqnarray}
Therefore, we can alternatively minimize the inner and outer problems to find the solutions of the original problem. Such decomposition allows us to disentangle the non-strong convexity and non-smoothness of the original problem. Specifically, the inner problem is strongly-convex and non-smooth, and the outer problem is convex and smooth. We can achieve uniform stability for both of the two problems. Below we provide the details.

\paragraph{Uniform Stability of Inner Minimization.} Based on the definition, $K(w,u;z)$ is strongly-convex. The inner minimization problem is a strongly-convex, non-smooth problem. The following Lemma shows that the minimizer of a strongly-convex problem is $\mathcal{O}(1/pn)$-uniformly stable.

\begin{lemma}[Uniform Stability of Inner Minimization]\label{l2}
Assume that $h\in\mathcal{H}$. Let $p>0$. For neighbouring $S$ and $S'$, we have
$$\|w(u;S)-w(u;S')\|\le 2L/(np).$$
\end{lemma}

\paragraph{Uniform Stability of Outer Minimization.} The outer problem is a $p$-smooth convex problem. It is due to the following properties of the Moreau envelope function $M(u,S)$.
\begin{lemmalist}[Smoothness and Convexity of Moreau Envelops Functions]
\label{l1}
Assume that $h\in\mathcal{H}$. Let $p>0$. Then, $M(u;S)$ satisfies
\begin{senenum}
\item The gradient of $M(u;S)$ is $\nabla_u M(u;S)=p(u-w(u;S))$.\label{l1_2}
\item \label{weakly1} $M(u;S)$ is convex in $u$.
\item  \label{gL1} $M(u;S)$ is $p$-gradient Lipschitz continuous.
\end{senenum}
\end{lemmalist}
The proof of Lemma \ref{l1} is due to \citep{Rockafellar1976-js} and also provided in Appendix A.1. Based on the $p$-smooth property, the $\mathcal{O}(T^q/n)$-uniform stability of the outer problem is achieved by running gradient descent on $M(u;S)$.

\begin{lemma}[Uniform Stability of Outer Minimization]
\label{thm:exact}
Assume $h$ is a convex, $L$-Lipschitz function. Let $u^T(S)$ and $u^T(S')$ be the outputs of running GD on $M(u;S)$ and $M(u;S')$, respectively,  with fixed stepsize $\alpha\leq 1/p$ for $T$ steps. Then, the generalization gap satisfies

\begin{equation}\label{ATbound}
\begin{aligned}
        \|u^T(S)-u^T(S')\| \leq \bigg(\frac{2L{\luo  T\alpha}}{n}\bigg).
\end{aligned}
\end{equation}
\end{lemma}

\textbf{Remark.} This upper bound is as tight as the result of running SGD on smooth convex finite-sum problems \citep{hardt2016train}. It is worth noticing that Lemma \ref{thm:exact} is not a corollary of the result of \citep{hardt2016train}, because $M(u,S)$ is not in the form of finite-sum. The proof of Lemma \ref{thm:exact} is based on Lemmas \ref{l2} and \ref{l1}. It is deferred to Appendix A.2.

\subsection{Uniform Stability of Moreau Envelope-$\mathcal{A}$}
\label{mea}
Let $\mathcal{A}$ be a first-order (stochastic) algorithm for the original problem $R_S(w)$, \emph{i.e.,} (BGD, SGD. In the equivalent problem, we directly apply $\mathcal{A}$ to the inner problem. Based on Lemma \ref{ATbound}, we apply GD to the outer problem. The algorithm is provided in Alg. \ref{alg:SSGDMax}).

\begin{algorithm}[h]
    \caption{Moreau Envelope-$\mathcal{A}$ (ME-$\mathcal{A}$)}
    \label{alg:SSGDMax}
\begin{algorithmic}[1]
\STATE Initialize $w^0$, $u^0$;
\STATE Choose stepsize $0<\alpha_t\leq 1/p$ (or let $\tau_t=1-\alpha_t p$, $0\leq\tau_t<1$);
\FOR{$t=0,1,2,\ldots,T$}
\STATE Let $w^t_0=w^t$;
\FOR{$s=0,1,2,\cdots,N$}
\STATE $w^{t}_{s+1}=\mathcal{A}(K(w^t_s,u^t;S))$;
\ENDFOR
\STATE $w^{t+1}=w^t_N$;
\STATE $u^{t+1}=u^t+\alpha_t p (w^{t+1}-u^t)$ (or $u^{t+1}=\tau_t u^t+(1-\tau_t) w^{t+1}$);
\ENDFOR
\end{algorithmic}
\end{algorithm}

To develop the uniform stability of ME-$\mathcal{A}$, we first assume the optimization error of the inner problem. Let
\begin{equation}
    \|w_N^t-w(u^t;S)\|\leq\varepsilon(\mathcal{A}),
\end{equation}
\emph{i.e.,} the distance between the output of the inner minimization problem $w_N^t$ and the optimal minimizer $w(u^t;S)$ in each iterations $t$ is at most $\varepsilon(\mathcal{A})$. For example, the convergence rate of running SGD on strongly convex, non-smooth minimization problem is $\mathcal{O}(1/N)$ \citep{nemirovskij1983problem}.

\begin{theorem}[Generalization bound of ME-$\mathcal{A}$ in Convex Case]
\label{thm:SSGDmax}
Assume that $h$ is convex and $L$-Lipschitz. Suppose we run ME-$\mathcal{A}$ with stepsize $\alpha_t\leq 1/p$ for $T$ steps. The generalization gap satisfies
\begin{equation}
\begin{aligned}
        \mathcal{E}_{gen} \leq L\bigg(\frac{2L}{n} +2p\varepsilon(\mathcal{A})\bigg)\sum_{t=1}^T\alpha_t.
\end{aligned}
\end{equation}
Furthermore, if algorithm $\mathcal{A}$ satisfies $\varepsilon(\mathcal{A})=\mathcal{O}(1/pn)$, the generalization gap satisfies
\begin{equation}
\begin{aligned}
\mathcal{E}_{gen} \leq \mathcal{O}\bigg(\frac{2L^2}{n}\bigg)\sum_{t=1}^T\alpha_t.
\end{aligned}
\end{equation}
\end{theorem}

\paragraph{Remark:} In strongly-convex, non-smooth problem, $\varepsilon(\mathcal{A})=\mathcal{O}(1/pn)$ can be achieved by$\mathcal{A}$=SGD with diminishing stepsize \citep{nemirovskij1983problem}, for more discussion, see Appendix B. When the stepsize is fixed to $\alpha_t=\alpha$, it is showed that Alg. \ref{alg:SSGDMax} achieves generalization bound of $\mathcal{O}(T\alpha/n)$.

 \subsection{Non-convex Case}
\label{sec:weakly}
Our principal contribution is focused on non-convex settings, encompassing a function class that extends beyond the limitations of convex scenarios.

In this setting, it seems that nothing can be done since the adversarial loss is both non-smooth and non-convex. Fortunately, the smoothness of the standard loss guarantee that the adversarial loss is weakly convex defined as followed.

\begin{definition}Let $l>0$. A function is said to be $-l$-weakly convex if $\forall x$, $f(x)+l\|x\|^2/2$ is convex in $x$.

This can be attributed to the following reasons: If the standard loss is smooth, meaning it has a gradient that is Lipschitz continuous, then the standard loss exhibits both upper and lower curvature. The presence of lower curvature implies that the standard loss is weakly convex, as lower curvature is equivalent to weak convexity. Furthermore, when the maximum operation is applied to the standard loss, the resulting adversarial loss retains this property of weak convexity.

\end{definition} In this case, we require $p>l$ such that $M(u;S)$ is strongly convex. Firstly, we extend Lemma \ref{l0} to \ref{l1} to weakly convex cases, which are Lemma A.1 to A.3 presented in Appendix. Based on the Lemma, we can derive the stability-based generalization bounds for ME-$\mathcal{A}$ in weakly-convex case.

\begin{theorem}[Generalization bound of ME-$\mathcal{A}$ in Weakly-Convex Case]
\label{weakexact}
Assume that $h$ is a weakly-convex, $L$-Lipschitz function. Suppose we run ME-$\mathcal{A}$ with diminishing stepsize $\alpha\leq 1/\beta t$ for $T$ steps, where $\beta = \max\{p,pl/(p-l)\}$. Then, the generalization gap satisfies
\begin{equation}
\begin{aligned}
\mathcal{E}_{gen} \leq \mathcal{O}\bigg(\frac{2L^2 T^q}{\min\{l,p-l\}n}+\varepsilon(\mathcal{A})2LpT^q/\beta\bigg),
\end{aligned}
\end{equation}
where $q=\beta c<1$. Furthermore, if algorithm $\mathcal{A}$ satisfies $\varepsilon(\mathcal{A})=\mathcal{O}(1/(p-l)n)$, the generalization gap satisfies
\begin{equation}
\begin{aligned}
\mathcal{E}_{gen} \leq \mathcal{O}\bigg(\frac{2L^2 T^q}{\min\{l,p-l\}n}\bigg).
\end{aligned}
\end{equation}

\end{theorem}

Notice that existing uniform stability algorithms on non-smooth loss requires convexity assumption. A main benefit of ME-$\mathcal{A}$ is that the algorithm can achieve $\mathcal{O}(T^q/n)$-uniform stability in weakly-convex cases. For comparison, by applying SGD to adversarial training, the robust generalization gap satisfies
$\mathcal{E}_{gen} \leq \mathcal{O}(2LT^q\epsilon+\frac{2L^2 T^q}{n})$.

\subsection{Convex Risk Minimization}
Now, we turn to the excess risk minimization analysis of ME-$\mathcal{A}$ in convex case. The excess risk can be decomposed to the sum of optimization and generalization error, \emph{i.e.,}
\begin{center}
	\emph{Excess risk =} $R_{\mathcal{D}}(w)-\min_{w\in W}R_{\mathcal{D}}(w)\leq\mathcal{E}_{opt}+\mathcal{E}_{gen}$.
\end{center}
\begin{theorem}[Optimization error of ME-$\mathcal{A}$ in Convex Case]
	\label{thm:opterr}
	Assume that $h$ is a convex, $L$-Lipschitz function. Suppose we run ME-$\mathcal{A}$ with stepsize $\alpha\leq 1/p$ for $T$ steps. Then, the optimization error satisfies
	\begin{equation*}
		\begin{aligned}
			\mathcal{E}_{opt} &\leq\mathcal{O}\bigg(\frac{\|u_0-u^*\|^2}{2\alpha T}+\varepsilon(\mathcal{A})p\frac{\sum_{t=1}^T\|u^t-u^*\|^2}{T}\bigg)\\
   &=\mathcal{O}\bigg(\frac{D_0}{2\alpha T}+\varepsilon(\mathcal{A}) p D_1\bigg).
		\end{aligned}
	\end{equation*}
	Furthermore, if algorithm $\mathcal{A}$ satisfies $\varepsilon(\mathcal{A})=\mathcal{O}(1/p \alpha T)$, the optimization gap satisfies
	$\mathcal{E}_{opt} \leq \mathcal{O}\big(\frac{1}{T\alpha}\big)$.
\end{theorem}

A minimax lower bound of the excess risk is given in \citep{nemirovskij1983problem}: $ \min_w\max_{\mathcal{D}}\mathbb{E}_{S}[R_{\mathcal{D}}(w)-\min_{w\in W}R_{\mathcal{D}}(w)]\geq \Omega(LD_W/\sqrt{n})$. By setting $T\alpha=\mathcal{O}(\sqrt{n})$, ME-$\mathcal{A}$ achieves the optimal excess risk with respect to $T$ and $\alpha$, resulting in $\text{Excess risk}\leq\mathcal{E}_{opt}+\mathcal{E}_{gen}\leq \mathcal{O}(\sqrt{1/n}$). This means that ME-$\mathcal{A}$, with an optimal stopping criterion, achieves the minimax lower bound in $\Omega (1/\sqrt{n})$ for excess risk. On the other hand, by combining the lower bound of excess risk and the upper bound of optimization error, we can obtain a lower bound in $\Omega(T\alpha/n)$ for the generalization gap.

\paragraph{Iteration Complexity.} To avoid introducing additional computational cost, we hope the number of outer iterations $T=\mathcal{O}(1)$. Then, $\alpha=\mathcal{O}(\sqrt{n})$. Based on the condition $\alpha\leq 1/p$, we should have $p\leq\mathcal{O}(1/\sqrt{n})$. In this setting, the iteration complexity of ME-$\mathcal{A}$ is the same as that of $\mathcal{A}$. For example, Let $\mathcal{A}$= SGD, based on Thm. \ref{thm:SSGDmax} and \ref{thm:opterr}, it is required $\varepsilon(\mathcal{A})=\mathcal{O}(1/pn)$ for achieving uniform stability and   $\varepsilon(\mathcal{A})=\mathcal{O}(1/p
\alpha T)=\mathcal{O}(1/p\sqrt{n})$ for achieving $\mathcal{O}(1/T\alpha)$ optimization error. Then, the number of iterations of both of ME-$\mathcal{A}$ and $\mathcal{A}$ are both $\mathcal{O}(n^2)$. For more details, see Appendix B.2. 

\section{Comparison of Moreau Envelope-$\mathcal{A}$ with Existing Algorithms}
\label{further}
\begin{table*}[htbp]
        \caption{Comparison of Moreau envelope-$\mathcal{A}$ (ME-$\mathcal{A}$) with stochastic weight averaging (SWA), empirical risk minimization, Moreau Yosida smoothing (MYS) and proximal point methods (PPM) on non-smooth minimization problem. Only Moreau envelope-$\mathcal{A}$ achieves uniform stability for weakly-convex non-smooth problems. }
    \label{Compare1}
    \centering
    \begin{tabular}{cccccc}
    \toprule
        \multirow{2}{*}{Algorithms}&\multirow{2}{*}{Types}&\multirow{2}{*}{Difference to ME-$\mathcal{A}$}&\multicolumn{2}{c}{Convex} & Weakly-Convex\\
            \cmidrule(rl){4-5}
        && &Stability &  Complexity & Stability\\
       \midrule
        SWA &-& $p=0$, $\tau>0$  & \XSolidBrush  &  - & \XSolidBrush\\
        \midrule
        ERM &\multirow{2}{*}{Regularization}& $u=0$ & \Checkmark  & -   & \XSolidBrush\\
        (Phase)-ERM && $u=0$, $p\rightarrow 0$   & \Checkmark & $\mathcal{O}(n^2)$  & \XSolidBrush\\
        \midrule
       MYS&\multirow{3}{*}{Moreau Envelope}& $\mathbb{E}_S\min_w[\cdot]\Rightarrow\min_w\mathbb{E}_S[\cdot]$ & \Checkmark  & $\mathcal{O}(n^{4.5})$& \XSolidBrush\\
      PPM (Our proof) & & $\tau =0$ &  \Checkmark    & $\mathcal{O}(n^2)$& \XSolidBrush\\
       \textbf{ME-$\mathcal{A}$ (Ours)}& & N/A &  \Checkmark    & $\mathcal{O}(n^2)$ &  \Checkmark\\
       \bottomrule
    \end{tabular}
\end{table*}

ME-$\mathcal{A}$ looks similar to but essentially different from some existing algorithms. It is necessary to compare ME-$\mathcal{A}$ with these similar algorithms in detail and the summary of the comparison is provided in Table \ref{Compare1}.
\vspace{-0.1in}
\paragraph{Stochastic Weight Averaging (SWA).} SWA \citep{izmailov2018averaging} (or moving averaging in different literatures) suggests using the weighted average of the iterates rather than the final one for inference. The update rules of SWA is $u^{t+1}=\tau^t u^t +(1-\tau^t)w^{t+1}$. By Simply applying the analytical tools in the work of \citep{bassily2020stability}, we can see that SWA is not guarantee to be uniformly stable. On the other hands, SWA can be regarded as ME-$\mathcal{A}$ when $p\rightarrow 0$. In Alg. \ref{alg:SSGDMax}, if we denote $\tau^t=1-\alpha^t p$, Step 9 can be view as a weight averaging step. In Thm. \ref{thm:SSGDmax}, it is required that $0<\alpha^t\leq 1/p$. Then, $0\leq\tau^t=(1-\alpha^tp)< 1 $. Therefore, by fixing $\alpha^tp$ to be constant and letting $p\rightarrow 0$, ME-$\mathcal{A}$ is reduced to SWA.
Our analysis provides an understanding of the generalization ability of SWA.
\vspace{-0.1in}
\paragraph{Empirical Risk Minimization (ERM).} ERM (or weight decay in different literatures) is a technique used to regularize the empirical loss by adding a $\ell_2$ regularization term, \emph{i.e.,} the loss function is $h(w;z)+p\|w\|^2/2.$ If we replace Step 9 with $u=0$ in Alg. \ref{alg:SSGDMax}, it is reduced to ERM. In convex case, the regularized loss becomes strongly convex and ensures uniform stability \citep{bousquet2002stability}. \citet{feldman2020private} introduced a variant of ERM called Phase-ERM and showed that it achieves uniform stability in $\mathcal{O}(n^2)$ steps.

However, such theoretical guarantee cannot be extended to weakly-convex problems. This limitation arises because the regularization term changes the global solutions of the problems, necessitating that the parameter $p$ cannot be too large. In weakly-convex case, we may not have $p>l$, implying that $h(w;z)+p\|w\|^2/2$ may not exhibit strong convexity. As a result, ERM is not guaranteed to be uniformly stable in weakly-convex scenario. 

\subsection{Comparison of Moreau Envelope-Type Algorithms}

\paragraph{Moreau Yosida-Smoothing (MYS).} MYS is originally proposed to solve non-smooth convex optimization problems \citep{nesterov2005smooth}. It uses the Moreau envelope function to smooth the non-smooth loss $h(w;z)$. Then, algorithms can be applied to the smooth surrogate. The loss of MYS is
\begin{equation*}
\begin{aligned}
    &\min_u\mathbb{E}_S\min_w  \big[h(w;z)+\frac{p}{2}\|w-u\|^2\big].
\end{aligned}
\end{equation*}
The uniform stability analysis of MYS is provided in (\citet{bassily2019private}, cf. Thm. 4.4). It is shown that $\mathcal{O}(n^{4.5})$ steps are required to achieve uniform stability, which is inefficient. In deep learning, memorization cost is another issue. It is required to store $n$ different networks (or (n/batch size) in BGD settings) in memory, which is intractable. Additionally, it is worth notice that the theoretical analysis of MYS and ME-$\mathcal{A}$ are different. The analysis of MYS use standard tools of applying SGD to finite sum smooth problems.
\vspace{-0.1in}
\paragraph{Proximal Point Methods (PPM).} PPM uses the proximal operator $w(u)$ to be the update rule, \emph{i.e.,} $u^{t+1}=w(u^t)$. Notice that it is equivalent to apply GD to the Moreau envelope function $M(u;S)$ with constant stepsize $\alpha_t=1/p$ (or $\tau=0$). Therefore, PPM can be regarded as a special case of Moreau envelope-$\mathcal{A}$ ($\alpha=1/p$, $\mathcal{A}$=global minimizer). In convex case, a by-product of our results is that PPM achieves uniform stability for non-smooth loss minimization. However, PPM is not guaranteed to be uniformly stable in weakly-convex case. In Thm. \ref{weakexact}, diminishing stepsize is required for achieving $\mathcal{O}(T^q/n)$-uniform stability, but PPM is equivalent to GD with fixed stepsize.

\section{Experiments}
\label{sec:exper}
\begin{figure*}[h]
	\centering
	\subfigure[]{
		\centering
 	\includegraphics[width=0.31\linewidth]{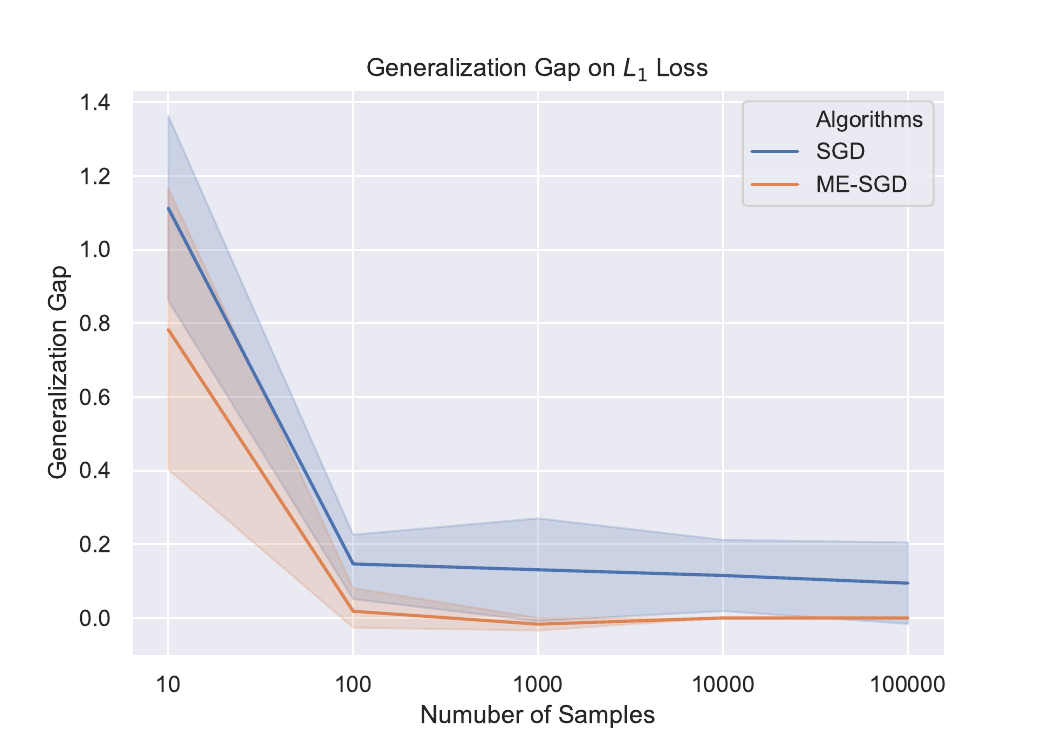}
	}
	\subfigure[]{
		\centering
		\includegraphics[width=0.31\linewidth]{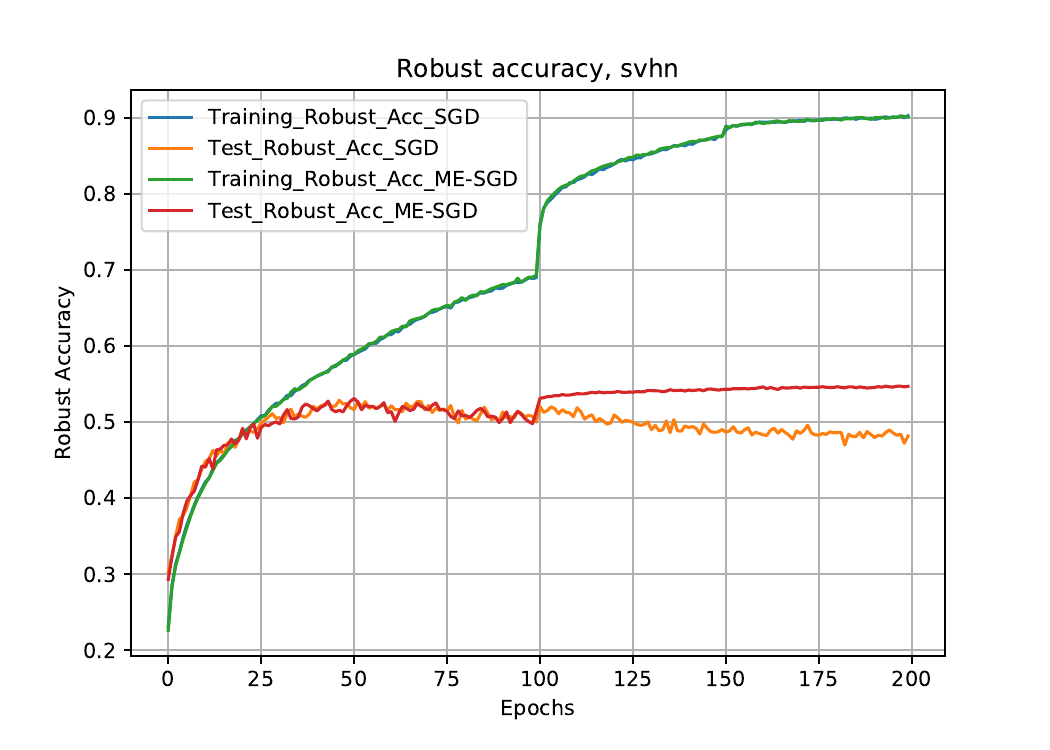}
	}
	\subfigure[]{
		\centering
		\includegraphics[width=0.31\linewidth]{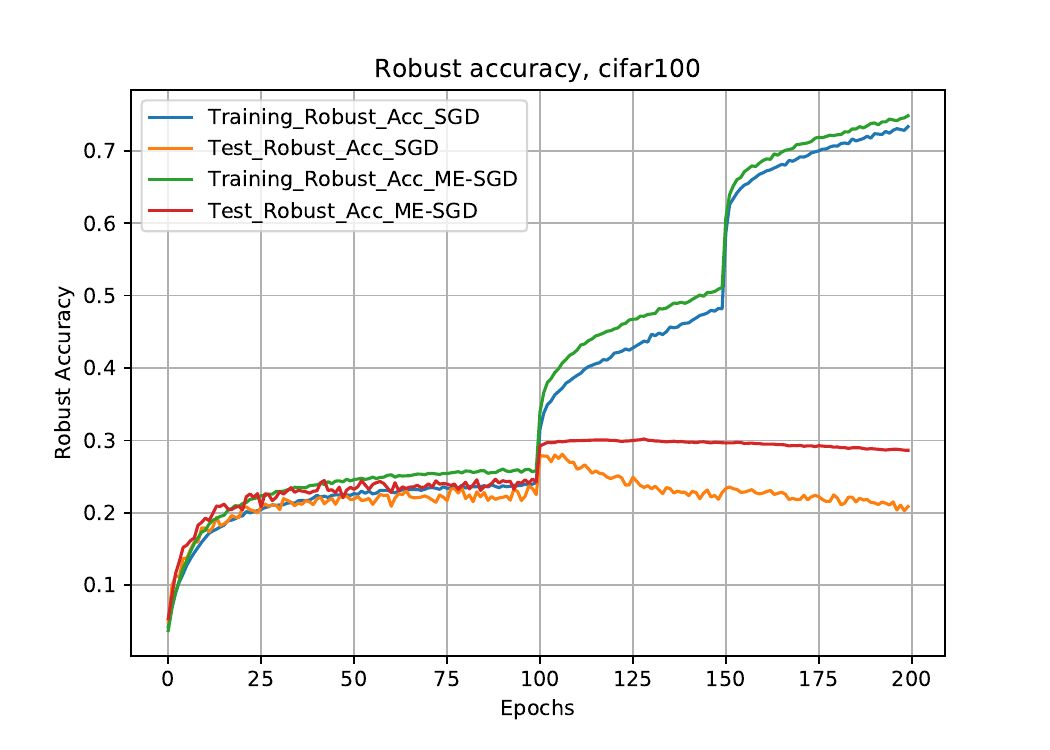}
	}
	\caption{(a): Comparison of generalization gap induced by SGD and ME-SGD for the toy example. (b) and (c): Robust test accuracy of adversarial training using SGD and ME-$\mathcal{A}$ on SVHN, and CFAR-100, respectively.}
	\label{fig:procedure}
\end{figure*}

We focus on $\mathcal{A}$=SGD in experiments to verify the theoretical results. We start from a toy experiment that perfectly matches the theory. Let the $L_1$ loss be $h(w,z)=\|w-z\|_1.$ It is 1-Lipschitz, since $ |h(w_1,z)-h(w_2,z)|=|\|w_1-z\|_1-\|w_2-z\|_1|\leq \|w_1-w_2\|_1,$ for all $z$. It is non-smooth due to the $\ell_1$-norm. Therefore, $L_1$-loss $\in\mathcal{H}$. 

Let the true distribution be a $10$-dimensions Gaussian distribution, $\mathcal{D}=\mathcal{N}(0,I)$. We sample 10 to 100000 data to train $w$ with SGD and ME-SGD. The results are shown in Figure \ref{fig:procedure} (a).  When the number of samples increases, the generalization gap induced by SGD does not converge to zero. While using ME-SGD, the gap converges to 0 fastly.
\subsection{Mitigating Robust Overfitting in $\mathcal{O}(T^q\epsilon)$}
Next, we turn our focus on adversarial training\footnote{Code is publicly available at \url{https://github.com/JiancongXiao/Moreau-Envelope-SGD}.}. We first study the robust overfitting issue. We mainly adopt the hyper-parameter settings of adversarial training in the work of \cite{gowal2020uncovering}. Weight decay is set to be $5\times 10^{-4}$. Based on Thm. \ref{weakexact}, the step size $\alpha_t$ of updating $u$ is set to be $1/pt$, then $\tau_t=1-\alpha_t p =(t-1)/t$. Interestingly, this theoretical-driven stepsize is consistent with the default setting of SWA stepsize used in practice. Ablation studies for the choice of $p$ and the choice of $\tau_t$ are provided in Appendix C.

To have a first glance at how ME-$\mathcal{A}$ mitigates robust overfitting, we consider the training procedure against $\ell_{\infty}$-PGD attacks \citep{madry2017towards} on SVHN, CIFAR-10, and CIFAR-100. For the attack algorithms, we use $\epsilon=8/255$. The attack step size is set to be $\epsilon/4$. We use piece-wise learning rates, which are equal to $0.1,0.01,0.001$ for epochs 1 to 100, 101 to 150, and 151 to 200, respectively. 

The experiments on CIFAR-10 is provided in Introduction. The experiments on SVHN and CIFAR-100 are provided in Figure \ref{fig:procedure} (b) and (c). It is shown that adversarially-trained models suffer from severe overfitting issues \citep{rice2020overfitting}. From the perspective of uniform stability, it might be due to the additional term $\mathcal{O}(T^q\epsilon)$ induced by SGD. For SGD, the robust test accuracy starts to decrease at around the $100^{th}$ epoch, which is a typical phenomenon of robust overfitting. 

By employing ME-$\mathcal{A}$, we significantly mitigate the issue of robust overfitting. Our experiments on SVHN and CIFAR-10/100 demonstrate that the robust test accuracy no longer diminishes. With this approach, DNNs fit the training adversarial examples well, reaching a performance ceiling that is constrained by the size of the existing dataset, denoted as $n$.
\subsection{Sample Complexity in $\mathcal{O}(T^q/n)$} 
\begin{figure}[h]
\centering
\subfigure[]{
\centering
\includegraphics[width=0.47\linewidth]{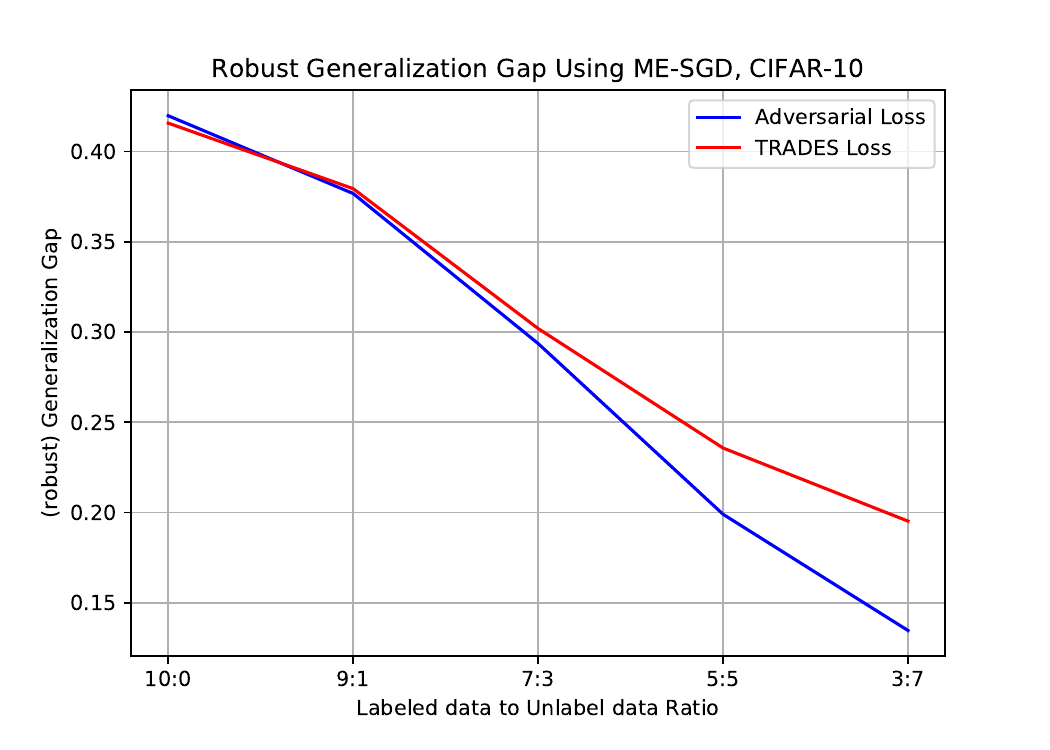}
}
\subfigure[]{
\centering
\includegraphics[width=0.47\linewidth]{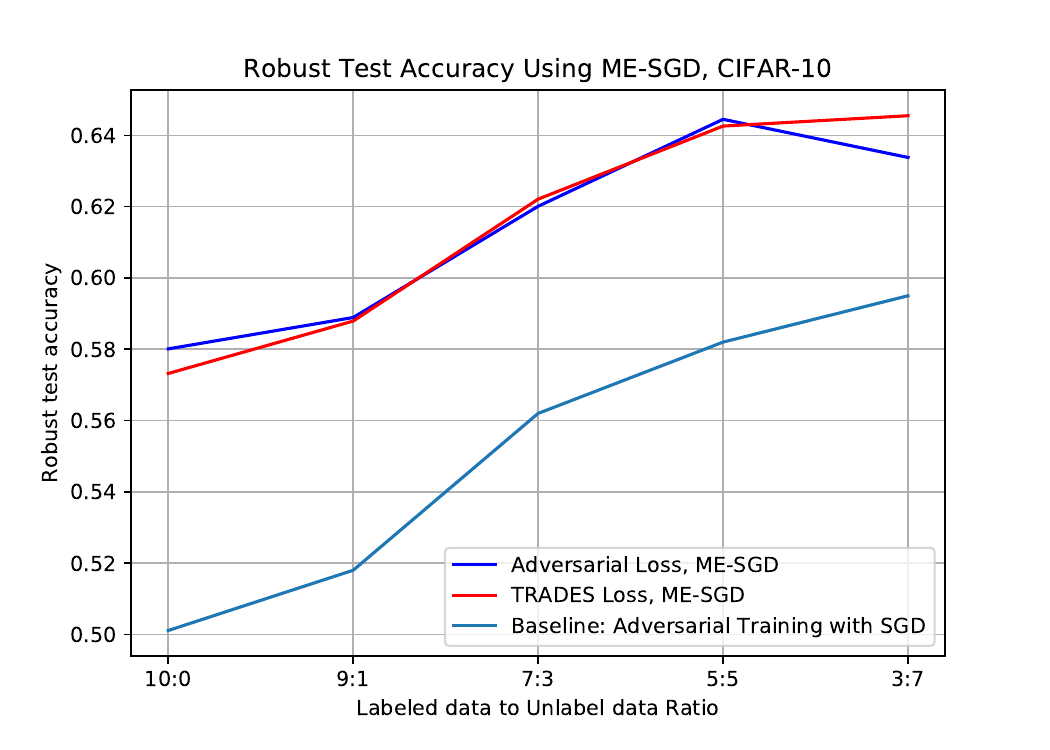}
}
\centering
\caption{\small{Robust generalization gap and robust test accuracy in the experiments of SGD and  ME-$\mathcal{A}$.}}
\label{fig:gap}
\end{figure}
Secondly, we examine the sample complexity characterized by $\mathcal{O}(T^q/n)$ as detailed in Theorem \ref{weakexact}. Considering that CIFAR-10 encompasses only 50,000 training samples, we utilize an additional pseudo-label dataset as introduced by \citet{carmon2019unlabeled} to analyze the sample complexity. Employing a greater proportion of pseudo-label data serves as a proxy for increasing the training dataset size. In Figure \ref{fig:gap}, we illustrate both the adversarial generalization gap (a) and the robust test accuracy (b). It is observed that by increasing the number of training data, denoted as $n$, the robust generalization gap decreases, subsequently enhancing the robust test accuracy.
\vspace{-0.1in}
\paragraph{Additivity of $\mathcal{O}(T^q\epsilon)$ and $\mathcal{O}(T^q/n)$.} In Table \ref{tab1}, we study the interaction between applying ME-$\mathcal{A}$ and adding additional data. The baseline performance of WideResNet-28$\times$10 on CIFAR-10 is reported in  \citep{gowal2020uncovering}. We adopt the AutoAttack \citep{croce2020reliable}, which is a collection of four attacks in default settings, to evaluate the performance of robust test errors. When applying the TRADES loss, the effect of ME-$\mathcal{A}$ is approximately 3\%, regardless of the presence of additional data. The benefit of incorporating extra data contributes around 8\% improvement, applicable to both SGD and ME-$\mathcal{A}$. A similar trend is observed with adversarial loss, with a small difference being that ME-$\mathcal{A}$ exerts a slightly lesser impact on adversarial loss, reducing to about 1\%. These experiments support the assertion that there is an additive relationship between robust overfitting and sample complexity. The enhancements gained from reducing robust overfitting and from addressing sample complexity can be cumulatively superimposed.

\begin{table}[htbp]
 \caption{Robust test accuracy on adversarial loss and TRADES loss. The term `+data' means addition unlabeled data with labeled to unlabeled data ratio 3:7 is added to the training set. The notion $\downarrow$ means the improvement mainly comes from reducing the underlying term.}
    \centering
    \resizebox{\linewidth}{!}{%
    \begin{tabular}{cccc}
        \toprule
    \multirow{2}{*}{Algorithm} & \multicolumn{2}{c}{AutoAttack} &\multirow{2}{*}{Stability} \\
    \cmidrule(rl){2-3}
     &  AT Loss & TRADES &  \\
    \midrule
    SGD &50.80{\scriptsize±0.23\%}& 51.91\%&$\mathcal{O}(T^q\epsilon+T^q/n)$\\
    +data  & 58.41{\scriptsize±0.25\%}& 59.45\%&$\mathcal{O}(T^q\epsilon+\underline{T^q/n}\downarrow)$\\
    \midrule
     ME-SGD  & 51.66{\scriptsize±0.16\%}& 55.23{\scriptsize±0.19\%} &$\mathcal{O}(\underline{T^q\epsilon}\downarrow+T^q/n)$\\
     +data  & 59.14{\scriptsize±0.18\%}& 62.76{\scriptsize±0.15\%}&$\mathcal{O}(\underline{T^q\epsilon+T^q/n}\downarrow)$\\
    \bottomrule
    \end{tabular}}
    \label{tab1}
\end{table}

\subsection{Discussion of Exisitng Algorithms}
In our final discussion, we study the empirical performance of methods listed in Table \ref{Compare1}. ERM, essentially weight decay, is already a common practice in deep learning. MYS, however, is not implementable due to memory constraints. Regarding PPM, its experimentation is part of the ablation study on the parameter $\tau$ provided in Appendix \ref{app:add}. When $\tau=0$, its performance similar to that of SGD. Notably, none of these methods can mitigate robust overfitting. In contrast, SWA demonstrates practical improvements in robust generalization, achieving robust test accuracy comparable to ME-$\mathcal{A}$, as exemplified in \citet{gowal2020uncovering}. The similarity between SWA and ME-$\mathcal{A}$ suggests that our theory may offer insights into SWA, particularly when considering $p\rightarrow 0$. However, it is important to note that while SWA shows practical promise, it lacks provable results, unlike ME-$\mathcal{A}$, which is theoretically underpinned. This theoretical assurance is a significant advantage of ME-$\mathcal{A}$ over SWA.

Recent studies, including \cite{rebuffi2021fixing}, indicate that robust generalization can be further enhanced by generating additional data using diffusion models. In our theoretical framework, the use of both pseudo-labeled data and generated data serves to optimize the terms $\mathcal{O}(T^q/n)$. In Table \ref{tab4}, we classify key existing methods for robust generalization into two distinct categories. This categorization helps validate our theoretical model.
\begin{table}[htbp]
 \caption{Existing approaches for improving adversarially robust generalization.}
    \centering
    \resizebox{\linewidth}{!}{%
    \begin{tabular}{cccc}
        \toprule
    Algorithms & Types & Empirical & Theoretical \\
    \midrule
    SWA &\multirow{2}{*}{Robust Overfitting}& \Checkmark&\XSolidBrush\\
    ME-$\mathcal{A}$  & & \Checkmark&\Checkmark\\
    \midrule
     Pseudo Labeled Data &\multirow{2}{*}{Sample Complexity}&\multicolumn{2}{c}{\Checkmark} \\
     Generated Data  &&\multicolumn{2}{c}{\Checkmark}\\
    \bottomrule
    \end{tabular}}
    \label{tab4}
\end{table}
\vspace{-0.1in}
\section{Conclusion}
In this paper, we present ME-$\mathcal{A}$, an approach aimed at achieving uniform stability for adversarial training and weakly convex non-smooth problems, with the goal of mitigating robust overfitting. One of the limitation is the weakly-convex assumption. While we have extended our results from the classical convex setting to a broader weakly convex framework, a gap still persists between weak convexity and the practical complexities of training deep neural networks (DNNs). Bridging this gap to encompass a more general subset of non-convex functions poses challenges. Furthermore, designing other new algorithms that are uniformly stable for non-convex non-smooth problems is a difficult task. In summary, we anticipate that ME-$\mathcal{A}$ will serve as an inspiration for the development of novel algorithms within the realm of deep learning.

\newpage
\section*{Potential Broader Impact}
This paper is primarily theoretical in nature, focusing on complex concepts and in-depth analysis. The potential impact lies in the possibility that our theoretical findings could inform future practical applications and advancements in the field. This could include improved algorithms, enhanced understanding of neural network behaviors, or even influencing the development of more robust systems in various technological domains. Such impacts, while emerging indirectly from our theoretical work, hold the promise of contributing significantly to both the academic community and real-world applications over time.
\bibliography{main.bib}
\bibliographystyle{icml2024}

\newpage
\appendix
\onecolumn
\section{Proof of Theorems}

\subsection{Proof of Lemma \ref{l0} and \ref{l1}}
\label{a1}
The proof of Lemma \ref{l0} and \ref{l1} are related. It suffices to prove the Lemma in weakly-convex case. We first state the Lemma in weakly-convex case.

\begin{lemmalist}
\label{la1}
Assume that $h$ is $-l$-weakly convex. Let $p>l$. Then, $M(u;S)$ satisfies
\begin{senenum}
\item $\min_u M(u;S)$ has the same global solution set as $\min_w R_S(w)$.
\item The gradient of $M(u;S)$ is $\nabla_u M(u;S)=p(u-w(u;S))$.
\item \label{weakly} $M(u;S)$ is $pl/(p-l)$-weakly convex in $u$.
\item  \label{gL} $M(u;S)$ is $\max\{p,pl/(p-l)\}$-gradient Lipschitz continuous.
\end{senenum}
\end{lemmalist}
Lemma \ref{l0} and \ref{l1} hold by letting $l=0$. To simplify the notation, we use $M(u)$ as a short hand notation of $M(u;S)$. Similar to $h(u)$, $K(u)$, and $w(u)$.

1. Let $w^*\in \argmin R_S(w)${\luo. We have}
\[R_S(w^*)=K(w^*,u=w^*,S)\geq K(w(u),u=w^*,S)\geq R_S(w(u=w^*)).\]
{\luo Then, the equality holds.} Therefore, $w=u=w^*$ is the optimal {\luo solution} of both {\luo $\min_w R_S(w)$ and $\min_uM(u;S)$.}

2. Since $K(w,u)$ is a ($p-l$)-strongly convex function, $w(u)$ is unique. Then
\[
M(u)=h(w(u))+\frac{p}{2}\|w(u)-u\|^2.
\]
{\luo By taking} the derivative of $M(u)$ with respect to $u$, we have
\begin{eqnarray}
\nabla_u M(u)&=&\bigg[\frac{\partial w(u)}{\partial u}\bigg]^T\cdot\nabla_{w(u)} h(w(u))+\bigg[\frac{\partial w(u)}{\partial u}-I\bigg]^T \cdot p(w(u)-u).\\
&=&\bigg[\frac{\partial w(u)}{\partial u}\bigg]^T\cdot (\nabla_{w(u)} h(w(u))+p(w(u)-u))+p(u-w(u)).\label{a7}
\end{eqnarray}
{\luo Since} $w(u)$ is the optimal solution of $K(w,u)$, we have
\begin{equation}\label{eq:firstorder}
    \nabla_{w(u)} K(w(u),u)= \nabla_{w(u)} h(w(u))+p(w(u)-u)=0.
\end{equation}
Therefore, the first term in \ref{a7} is equal to zero. We have $\nabla_u M(u)=p(u-w(u))$.

3. In Eq. (\ref{eq:firstorder}), take the derivatives with respect to $u$ on both sides, we have
\begin{equation}
  \bigg[\frac{\partial w(u)}{\partial u}\bigg]^T\nabla_w^2 h(w)+p(\bigg[\frac{\partial w(u)}{\partial u}\bigg]^T-I)=0.
\end{equation}
Organizing the terms, we have
\begin{equation}
  \bigg[\frac{\partial w(u)}{\partial u}\bigg]^T(\nabla_w^2 h(w)+pI)=pI.
\end{equation}
Since $h(w)$ is $-l$-weakly convex, $\nabla_w^2 h(w)+pI$ is positive definite. Then,
\begin{equation}\label{wu}
  \bigg[\frac{\partial w(u)}{\partial u}\bigg]^T\prec\frac{p}{p-l} I.
\end{equation}
Then,
\begin{equation}
    \nabla_u^2 M(u)=[\frac{\partial}{\partial u}p(u-w(u))]^T=p (I-\bigg[\frac{\partial w(u)}{\partial u}\bigg]^T) \succ p(1-\frac{p}{p-l})I.
\end{equation}
Therefore, $M(u)$ is a $pl/(p-l)$-weakly convex function.

4. By Eq. (\ref{wu}), we have
\begin{eqnarray}
        &&\|\nabla M(u_1)-\nabla M(u_2)\|\nonumber\\
        &=&p \|u_1-w(u_1)-u_2-w(u_2)\|\nonumber\\
        &\leq & p \max\{\|u_1-u_2\|,\|w(u_1)-w(u_2)\|\}\nonumber\\
        &\leq & p\max\{1,l/(p-l)\}\|u_1-u_2\|\nonumber\\
        &= & \max\{p,pl/(p-l)\}\|u_1-u_2\|\nonumber.
\end{eqnarray}
Therefore, $M(u;S)$ is $\max\{p,pl/(p-l)\}$-gradient Lipschitz continuous.
\qed

\subsection{Proof of Lemma \ref{l2} and \ref{thm:exact}}
Now we discuss the proof of the uniform stability of inner and outer minimization.
\label{proof51}

Proof of Lemma \ref{l2}: It suffices to prove the Lemma in weakly-convex case. We first introduce the following Lemma of uniform stability of inner minimization.
\begin{lemma}[Uniform Stability of Inner Minimization in Weakly-Convex Case]\label{A1}
In weakly-convex case, for neighbouring $S$ and $S'$, we have
$$\|w(u;S)-w(u;S')\|\le 2L/(n(p-\ell)).$$
\end{lemma}
\begin{proof}
By the $(p-l)$-strongly convexity of $K(w,u;S)$, we have
\begin{eqnarray}
&&(p-l)\|w(u;S)-w(u;S')\|\nonumber\\
&\leq& \|\nabla K(w(u;S),u;S)-\nabla K(w(u;S'),u;S)\| \nonumber\\
&\leq& \|\nabla K(w(u;S),u;S)-\nabla K(w(u;S'),u;S')\|\nonumber\\
&&+\frac{1}{n}\|\nabla h(w(u;S'),z_i)\|+\frac{1}{n}\|\nabla h(w(u;S'),z_i')\|\nonumber\\
&=&\frac{1}{n}\|\nabla h(w(u;S'),z_i)\|+\frac{1}{n}\|\nabla h(w(u;S'),z_i')\|\nonumber\\
&\leq&\frac{2L}{n},\nonumber
\end{eqnarray}
where the second inequality is due to the definition of $K(w,u;S)$, the third one is due to the first-order optimally condition, and the last inequality is because of the bounded gradient of $h(w;z)$.
\end{proof}

Next, we move to the proof of Lemma \ref{thm:exact}. Lemma \ref{thm:exact} is not obtained from the work of \citep{hardt2016train}. Notice that
\[
M(u;S)=\min_{w\in W}\frac{1}{n}\sum_{z\in S}K(w,u;z)\neq \frac{1}{n}\sum_{z\in S}\min_{w\in W}K(w,u;z).
\]
$\min_uM(u;S)$ is not a finite sum problem. The analysis in \citep{hardt2016train} can only be applied to finite sum problems. Lemma \ref{thm:exact} requires a different proof. In summary, there are two steps:
\begin{enumerate}
    \item Build the recursion from $\|u_S^{t}-u_{S'}^{t}\|$ to $\|u_S^{t+1}-u_{S'}^{t+1}\|$;
    \item Unwind the recursion.
\end{enumerate}
The main difference comes from the first step. 

\paragraph{Step 1.}
\begin{eqnarray}
&&\|u_S^{t+1}-u_{S'}^{t+1}\|\nonumber\\
&=& \|u_S^{t}-u_{S'}^{t}-\alpha_t(\nabla M(u_S^t;S)-\nabla M(u_{S'}^t;S'))\|\nonumber\\
&\leq& \|u_S^{t}-u_{S'}^{t}-\alpha_t(\nabla M(u_S^t;S)+\nabla M(u_{S'}^t;S))\|+\alpha_t\|\nabla M(u_{S'}^t;S')-\nabla M(u_{S'}^t;S)\|\nonumber\\
&\leq& \|u_S^{t}-u_{S'}^{t}\|+\alpha^t\|\nabla M(u_{S'}^t;S')-\nabla M(u_{S'}^t;S)\|\nonumber\\
&=&  \|u_S^{t}-u_{S'}^{t}\|+\alpha^t p\|u_{S'}^t-u_{S'}^t-w(u_{S'}^t,S)+w(u_{S'}^t,S')\|\nonumber\\
&\leq&  \|u_S^{t}-u_{S'}^{t}\|+
\frac{2L\alpha_t}{n},\nonumber
\end{eqnarray}
{\luo where the second inequality is due to the non-expansive property of convex function , the last inequality is due to Lemma \ref{A1}. 

\paragraph{Step 2.} Unwinding} the recursion, we have
\[
\|u_S^{T}-u_{S'}^{T}\|\leq \frac{2L\sum_{t=1}^T\alpha_t}{n}.
\]\qed

\subsection{Proof of Thm. \ref{thm:SSGDmax}}
\begin{proof}
We decompose $\|u_S^{t+1}-u_{S'}^{t+1}\|$ as
\begin{eqnarray}\label{decomposition}
&&\mathbb{E}\|u_S^{t+1}-u_{S'}^{t+1}\|\nonumber\\
&=& \mathbb{E}\|u_S^{t}-\alpha_t\nabla
_u K(w_{N,S}^t,u_S^{t};S)-u_{S'}^{t}+\alpha_t\nabla_u K(w_{N,S'}^t,u_{S'}^{t};S')\|\nonumber\\
&\leq& \mathbb{E}\|u_S^{t}-\alpha_t\nabla
_u M(u_S^{t};S)-u_{S'}^{t}+\alpha_t\nabla_u M(u_{S'}^{t};S')\|\nonumber\\
&+&2\alpha_t \mathbb{E}\|\nabla
_u K(w_{N,S}^t,u_S^{t};S)-\nabla
_u M(u_S^{t};S)\|\nonumber\\
&\leq& \mathbb{E}\|u_S^{t}-u_{S'}^{t}\|+
\frac{2L\alpha_t}{n}+2\alpha_t p\mathbb{E}\|w_N^t-w(u^t)\|\nonumber\\
&\leq& \mathbb{E}\|u_S^{t}-u_{S'}^{t}\|+
\frac{2L\alpha_t}{n}+2\alpha_t p\varepsilon(\mathcal{A}).\nonumber
\end{eqnarray}
Unwind the recursion and let $u^T$ be the output of the algorithm. We have
\begin{eqnarray}
\mathcal{E}_{gen}&\leq& L\bigg(\frac{2L}{n} +2p\varepsilon(\mathcal{A})\bigg)\sum_{t=1}^T\alpha_t.\nonumber
\end{eqnarray}
{\luo If we choose $w(u^T)$ to be the algorithm output, we have}
\begin{eqnarray}
\mathcal{E}_{gen}&\leq& L\mathbb{E}\|w(u_S^{T};S)-w(u_{S'}^{T};S')\|\nonumber\\
&=& 
L\mathbb{E}\|u_S^{T}-\frac{1}{p}\nabla M(u_S^{T},S)-u_{S'}^{T}-\frac{1}{p}\nabla M(u_{S'}^{T};S'))\|\nonumber\\
&\leq& L\mathbb{E}\|u_S^{T}-u_{S'}^{T}\| + \frac{2L^2}{np}\nonumber\\
&\leq& L\bigg(\frac{2L}{n} +2p\varepsilon(\mathcal{A})\bigg)\sum_{t=1}^T\alpha_t + \frac{2L^2}{np}.\nonumber\\
\end{eqnarray}
{\luo where the first equality is due to $\nabla M(u;S)=p(u-w(u))$, the second inequality is due to the non-expansive propertiy of $M(u;S)$.}
\end{proof}

\subsection{Proof of Thm. \ref{weakexact}}
\label{sec:weakly2}

The proof contains two steps. The first step is to build the recursion, which is based on the error bound and the decomposition of Lemma \ref{thm:exact}. The second step is to unwind the recursion, which is adopt from the analysis of uniform stability in non-smooth case \citep{hardt2016train}.

Proof:

\paragraph{Step 1.}

We decompose $\|u_S^{t+1}-u_{S'}^{t+1}\|$ as

\begin{eqnarray}
&&\|u_S^{t+1}-u_{S'}^{t+1}\|\nonumber\\
&=& \mathbb{E}\|u_S^{t}-\alpha_t\nabla
_u K(w_{N,S}^t,u_S^{t};S)-u_{S'}^{t}+\alpha_t\nabla_u K(w_{N,S'}^t,u_{S'}^{t};S')\|\nonumber\\
&\leq& \mathbb{E}\|u_S^{t}-\alpha_t\nabla
_u M(u_S^{t};S)-u_{S'}^{t}+\alpha_t\nabla_u M(u_{S'}^{t};S')\|\nonumber\\
&+&2\alpha_t \mathbb{E}\|\nabla
_u K(w_{N,S}^t,u_S^{t};S)-\nabla
_u M(u_S^{t};S)\|\nonumber\\
&\leq& \|u_S^{t}-u_{S'}^{t}-\alpha_t(\nabla M(u_S^t;S)-\nabla M(u_{S'}^t;S'))\|+2\alpha_t p\varepsilon(\mathcal{A})\nonumber\\
&\leq& \|u_S^{t}-u_{S'}^{t}-\alpha_t(\nabla M(u_S^t;S)+\nabla M(u_{S'}^t;S))\|+\alpha_t\|\nabla M(u_{S'}^t;S')-\nabla M(u_{S'}^t;S)\|+2\alpha_t p\varepsilon(\mathcal{A})\nonumber\\
&\leq& \|u_S^{t}-u_{S'}^{t}\|+\alpha_t\|\nabla M(u_S^t;S)-\nabla M(u_{S'}^t;S)\|+\alpha^t\|\nabla M(u_{S'}^t;S')-\nabla M(u_{S'}^t;S)\|+2\alpha_t p\varepsilon(\mathcal{A})\nonumber\\
&\leq& (1+\alpha_t\beta)\|u_S^{t}-u_{S'}^{t}\|+\alpha^t\|\nabla M(u_{S'}^t;S')-\nabla M(u_{S'}^t;S)\|+2\alpha_t p\varepsilon(\mathcal{A})\label{b2},
\end{eqnarray}
 where the first inequality is due to triangular inequality, the second inequality is due to the assumption of the convergence rate of the inner minimization problem, and the third and fourth inequalities are due to triangular inequality. The last inequality is due to the gradient Lipschitz of $M(u;S)$ and $\beta=\max\{p,pl/p-l\}$. Then,
\begin{eqnarray}
&&\alpha^t\|\nabla M(u_{S'}^t;S')-\nabla M(u_{S'}^t;S)\|\nonumber\\
&=& \alpha^t p\|u_{S'}^t-u_{S'}^t-w(u_{S'}^t,S)+w(u_{S'}^t,S')\|\nonumber\\
&\leq&
\frac{2Lp\alpha_t}{(p-l)n}\label{b3},
\end{eqnarray}
where the first inequality is due to the form of $\nabla M(u;S)$, the last equality is due to Lemma \ref{A1}. 

Combining Eq. (\ref{b2}) and (\ref{b3}), we have
\begin{eqnarray}
&&\|u_S^{t+1}-u_{S'}^{t+1}\|\nonumber\\
&\leq& (1+\alpha_t\beta)\|u_S^{t}-u_{S'}^{t}\|+\frac{2Lp\alpha_t}{(p-l)n}+2\alpha_t p\varepsilon(\mathcal{A}).
\end{eqnarray}
\paragraph{Step 2.} Let $S$ and $S'$ be two samples of size $n$ differing in only a single
example. Consider two trajectories $w_1^1,\dots,w_1^T$ and $w_2^1,\dots,w_2^T$
induced by running SGD on sample $S$ and $S',$ respectively. Let $\delta_t=\|w_1^t-w_2^t\|$. Let $t_0\in\{0,1,\dots,n\},$ be the iteration that $\delta_{t_0}=0$, but SGD picks two different samples form $S$ and $S'$ in iteration $t_0+1$, then
\begin{equation}
\label{eq:nonconvex-diff}
    \mathcal{E}_{gen}\le \frac{t_0}n B
+ L\E\left[\delta_T\mid\delta_{t_0}=0\right]\,.
\end{equation}
Let $\Delta_t=\E\left[\delta_t\mid\delta_{t_0}=0\right]$, and $\alpha_t\leq c/(\beta t)$. Then, 

\begin{align*}
\Delta_{t+1}
& \le (1+\alpha_t\beta)\Delta_t +
\bigg(\frac{2 Lp}{(p-l)n}+2p\varepsilon(\mathcal{A})\bigg)\alpha_t\\
& = \left(1 + \frac{c\beta}t\right)\Delta_t + \bigg(\frac{2 Lp}{(p-l)n}+2p\varepsilon(\mathcal{A})\bigg)\frac{c}{t}\\
& \le \exp\left(\frac{c\beta}t\right)\Delta_t +\bigg(\frac{2 Lp}{(p-l)n}+2p\varepsilon(\mathcal{A})\bigg)\frac{c}{t}\,.
\end{align*}
Here we used the fact that $1+x\le\exp(x)$ for all $x.$

Using the fact that $\Delta_{t_0}=0,$ we can unwind this recurrence relation
from $T$ down to $t_0+1.$ This gives
\begin{align*}
\Delta_T &\leq \sum_{t=t_0+1}^T \left\{\prod_{k=t+1}^T
\exp\left(\tfrac{\beta c}{k}\right) \right\} \bigg(\frac{2 Lp}{(p-l)n}+2p\varepsilon(\mathcal{A})\bigg)\frac{c}{t}\\
&= \sum_{t=t_0+1}^T \exp\left( \beta c \sum_{k={t+1}}^T
\tfrac{1}{k} \right) \bigg(\frac{2 Lp}{(p-l)n}+2p\varepsilon(\mathcal{A})\bigg)\frac{c}{t}\\
&\leq \sum_{t=t_0+1}^T \exp\left( \beta c \log(\tfrac{T}{t})
\right) \bigg(\frac{2 Lp}{(p-l)n}+2p\varepsilon(\mathcal{A})\bigg)\frac{c}{t} \\
& =  \bigg(\frac{2 Lp}{(p-l)n}+2p\varepsilon(\mathcal{A})\bigg)c T^{\beta c } \sum_{t=t_0+1}^T t^{-\beta c - 1}\\
&\leq  \bigg(\frac{2 Lp}{(p-l)n}+2p\varepsilon(\mathcal{A})\bigg)\frac1{\beta c} c \left(\frac{T}{t_0}\right)^{\beta c } \\
&\le \frac{p}{\beta}\bigg(\frac{2L}{(p-l) n}+2\varepsilon(\mathcal{A})\bigg) \left(\frac{T}{t_0}\right)^{\beta c}\,,
\end{align*}
Plugging this bound into~\eqref{eq:nonconvex-diff},
we get
\[
\mathcal{E}_{gen} \le \frac{Bt_0}{n} +
\frac{Lp}{\beta}\bigg(\frac{2L}{(p-l) n}+2\varepsilon(\mathcal{A})\bigg) \left(\frac{T}{t_0}\right)^{\beta c}\,.
\]
Let $q = \beta c$. We select
t0 to optimize the right hand side
\[
\mathcal{E}_{gen} \le \frac{Bt_0}{n} +
\frac{Lp}{\beta}\bigg(\frac{2L}{(p-l) n}+2\varepsilon(\mathcal{A})\bigg) \left(\frac{T}{t_0}\right)^{\beta c}\,.
\]

If algorithm $\mathcal{A}$ satisfies $\varepsilon(\mathcal{A})=\mathcal{O}(1/(p-l)n)$, 
\begin{equation*}
\mathcal{O}\bigg(\frac{2LpT^{\beta c}\varepsilon(\mathcal{A})}{\beta}\bigg)=\mathcal{O}\bigg(\frac{2LpT^{\beta c}}{\beta(p-l)n}\bigg)=\mathcal{O}\bigg(\frac{2L T^{\beta c}}{\min\{l,p-l\}n}\bigg)
\end{equation*}
Let $q=\frac{\beta c}{1+\beta c}$, Optimize over $t_0$. Then, we have
\begin{equation}
\begin{aligned}
\mathcal{E}_{gen} \leq \mathcal{O}\bigg(\frac{2L^2 T^q}{\min\{l,p-l\}n}\bigg).
\end{aligned}
\end{equation}
\qed

\subsection{Proof of Thm. \ref{thm:opterr}}

Now we consider the convergence rate of $ME-\mathcal{A}$. In this part, we simply let $M(u)$ and $K(w,u)$ as short hand notations of $M(u;S)$ and $K(w,u;S)$.

By the $p$-gradient Lipschitzness of $M(u)$, we have

\begin{equation}
    M(u^{t+1})\leq M(u^t)+\langle\nabla M(u^t),u^{t+1}-u^t\rangle+\frac{p}{2}\|u^{t+1}-u^t\|^2.
\end{equation}
By the update rules of $u^t$, we have $u^{t+1}=u^t-\alpha_t \nabla K(w^{t+1},u^t)$ or $u^{t+1}=u^t+\alpha_t p (w^{t+1}-u^t)$. Then,
\begin{equation}
    M(u^{t+1})\leq M(u^t)-\alpha_t\langle\nabla M(u^t),\nabla K(w^{t+1},u^t)\rangle+\frac{p\alpha_t^2}{2}\|\nabla K(w^{t+1},u^t)\|^2.
\end{equation}

Based on the convexity of $M(u)$, we have
\begin{equation}
    M(u)\leq M(u^*)+\langle\nabla M(u),u-u^*\rangle.
\end{equation}
Then
\begin{eqnarray}
   &&M(u^{t+1})-M(u^*)\nonumber\\
   &\leq& \langle\nabla M(u^t),u^t-u^*\rangle-\alpha_t\langle\nabla M(u^t),\nabla K(w^{t+1},u^t)\rangle+\frac{p\alpha_t^2}{2}\|\nabla K(w^{t+1},u^t)\|^2\nonumber\\
    &=& \langle\nabla M(u^t), u^t-u^*-\alpha_t \nabla K(w^{t+1},u^t)\rangle+\frac{p\alpha_t^2}{2}\|\nabla K(w^{t+1},u^t)\|^2\nonumber\\
    &=& \langle\nabla M(u^t), u^{t+1}-u^*\rangle+\frac{p\alpha_t^2}{2}\|\nabla K(w^{t+1},u^t)\|^2\nonumber\\
    &=& \langle\nabla K(w^{t+1},u^t), u^{t+1}-u^*\rangle+\frac{p\alpha_t^2}{2}\|\nabla K(w^{t+1},u^t)\|^2\nonumber\\
    &&+\langle\nabla M(u^t)-\nabla K(w^{t+1},u^t), u^{t+1}-u^*\rangle.\label{eq:last}
\end{eqnarray}
Let $\alpha_t=\alpha\leq 1/p$. The sum of the first two terms is bounded by $\frac{1}{2\alpha}[\|u^t-u^*\|^2-\|u^{t+1}-u^*\|^2]$, it is because
\begin{eqnarray}
    &&\frac{1}{2\alpha}[\|u^t-u^*\|^2-\|u^{t+1}-u^*\|^2]\nonumber\\
    &=&\frac{1}{2\alpha}[2\alpha\langle\nabla K(w^{t+1},u^t),u^t-u^*\rangle-\alpha^2\|\nabla K(w^{t+1},u^t)\|^2]\nonumber\\
    &=&\langle\nabla K(w^{t+1},u^t),u^t-u^*\rangle-\frac{\alpha}{2}\|\nabla K(w^{t+1},u^t)\|^2\nonumber\\
    &=&\langle\nabla K(w^{t+1},u^t),u^{t+1}-u^*\rangle+\frac{\alpha}{2}\|\nabla K(w^{t+1},u^t)\|^2\nonumber\\
    &\geq&\langle\nabla K(w^{t+1},u^t),u^{t+1}-u^*\rangle+\frac{p \alpha^2}{2}\|\nabla K(w^{t+1},u^t)\|^2.\nonumber\\
\end{eqnarray}
The Last term in Eq. \ref{eq:last}
\begin{equation}
    \langle\nabla M(u^t)-\nabla K(w^{t+1},u^t), u^{t+1}-u^*\rangle\leq p\|w^{t+1}-w(u^t)\|\cdot\|u^{t+1}-u^*\|.
\end{equation}
Then, we have 
\begin{eqnarray}
   &&M(u^{t+1})-M(u^*)\nonumber\\
    &&\frac{1}{2\alpha}[\|u^t-u^*\|^2-\|u^{t+1}-u^*\|^2]+\varepsilon(\mathcal{A})p\|u^{t+1}-u^*\|.\nonumber\\
\end{eqnarray}
Sum over $t=0$ to $T$, we obtain
	\begin{equation}
		\begin{aligned}
			\mathcal{E}_{opt} \leq \mathcal{O}\bigg(\frac{\|u_0-u^*\|^2}{2\alpha T}+\varepsilon(\mathcal{A})p\frac{\sum_{t=1}^T\|u^t-u^*\|^2}{T}\bigg)=\mathcal{O}\bigg(\frac{D_0}{2\alpha T}+\varepsilon(\mathcal{A}) p D_1\bigg).
		\end{aligned}
	\end{equation}
	Furthermore, if algorithm $\mathcal{A}$ satisfies $\varepsilon(\mathcal{A})=\mathcal{O}(1/p \alpha T)$, the generalization gap satisfies
	\begin{equation}
		\begin{aligned}
			\mathcal{E}_{opt} \leq \mathcal{O}\bigg(\frac{1}{T\alpha}\bigg).
		\end{aligned}
	\end{equation}
	\qed

\section{Discussion on the Complexity of the Inner Problem}
\subsection{Convergence Rate of Strongly Convex Problems}
In this section, we consider the convergence rate of the inner minimization problem. Notice that the inner problem is a $(p-l)$-strongly convex, non-smooth problem.

\paragraph{Lipschitz Case.} We first consider the case that $K (w,u;z)$ is Lipschitz. If we assume that the domain $W$ of $w$ is  bounded by $D_W/2$, since
\[
\|\nabla_w K (w,u;z)\| = \|\nabla_w h(w;z)+p(w-u)\|\leq L+pD_W,
\]
$K (w,u;z)$ has bounded gradient $L_K=L+pD_W$, \emph{i.e.,} $ K (w,u;z)$ is $L_K$-Lipschitz. Then the convergence rate of running SGD on the inner problem can be obtained from classical strongly-convex optimization results.
\begin{lemma}\label{sc-nonsmooth2}
Given $t$ and $u^t$, suppose we run SGD on $K(w,u^t,S)$ w.r.t. $w$ with stepsize $c_s^t\leq 1/(p-l)s$ for $N$ steps. $w_N^{t}$ is approximately the minimizer with an error $C_1^2/N$, i.e.,
$$E\|w_N^{t}-w(u^t)\|^2\le \frac{C_1^2}{N},$$
where $C_1= (L+pD_W)/(p-l)$.
\end{lemma}

 By \citep{nemirovski2009robust}, running SGD on $K (w,u;S)$ with stepsize $c_s\leq 1/s(p-l)$ iccurs an optimization error in
$$E\|w_N-w(u)\|^2\le \frac{C_1^2}{N},$$
where $C_1= (L+pD_W)/(p-l)$. Then, by Jensen's inequality, we have
$$\epsilon(SGD)=E\|w_N-w(u)\|\leq\sqrt{E\|w_N-w(u)\|^2}\le \frac{C_1}{\sqrt{N}},$$

\paragraph{Non-Lipshitz Case.} Without the $D_W/2$-bounded assumption, $K(w,u^t,S)$ is not Lipshitz, but the sum of a Lipshitz function ($h(w)$) and a gradient Lipshitz function ($\frac{p}{2}\|w-u\|^2$). We provide our convergence analysis in this case. In this case, we apply subgradient method for this strongly convex problem. Consider the following optimization problem:
\begin{equation}\label{P}
\min_{w}\phi(w):=h(w)+\frac{p}{2}\|w-u\|^2,
\end{equation}
where $h(\cdot)$ is convex, $L$-Lipschitz and has bounded subgradients.
We perform subgradient descent to $\phi$:
$$w^{t+1}=w^t-cg_h(w^t)-cp(w^t-u).$$
\begin{theorem}
We can obtain $\|w^t-w^*\|\leq \epsilon$ with $t\leq \tilde{\mathcal{O}}(1/\epsilon^2)$ steps, where the constant hidden in $\mathcal{O}$ only depends on $L$, $\|w^0-w^*\|$ and $p$.
\end{theorem}
\begin{proof}
We have
\begin{eqnarray}
&&\|w^{t+1}-w^*\|^2\\
&=&\|w^t-cg^t-cp(w^t-u)-w^*\|^2\\
&=&\|(w^t-w^*)-c(g^t-g^*)-cp(w^t-w^*)\|^2,
\end{eqnarray}
where the last equality uses the optimality condition in $w^*$ that $g^*+p(w^*-u)=0$ for some $g^*\in \partial (h(w^*))$.
We then have
\begin{eqnarray}
&&\|(w^t-w^*)-c(g^t-g^*)-cp(w^t-w^*)\|^2\\
&=&(1-cp)\|w^t-w^*\|^2+c^2\|g^t-g^*\|^2\\
&-&(g^t-g^*)^T(w^t-w^*)\\
&\le&(1-cp)\|w^t-w^*\|^2+c^2L^2,
\end{eqnarray}
where the last inequality uses the fact that $\|g^t\|,\|g^*\|\leq L$ and $(g^t-g^*)^T(w^t-w^*)\geq 0$.
Let $\Delta_t=\|w^t-w^*\|^2$. Then we have the recursion 
$$\Delta_{t+1}\le (1-cp)\Delta_t+c^2L^2.$$
We have
$$(\Delta_{t+1}-cL^2/p)\le (1-cp)(\Delta_t-cL^2/p).$$
Take $c=\epsilon/2$ and we can attain the result.
\end{proof}

\subsection{Iteration Complexity} 
To avoid introducing additional computational cost, we hope the number of outer iterations $T=\mathcal{O}(1)$. Then, $\alpha=\mathcal{O}(\sqrt{n})$. Based on the condition $\alpha\leq 1/p$, we should have $p\leq\mathcal{O}(1/\sqrt{n})$. In this setting, the iteration complexity of ME-$\mathcal{A}$ is the same as that of $\mathcal{A}$. For example, Let $\mathcal{A}$= SGD, based on Thm. \ref{thm:SSGDmax} and \ref{thm:opterr}, it is required $\varepsilon(\mathcal{A})=\mathcal{O}(1/pn)$ for achieving uniform stability and   $\varepsilon(\mathcal{A})=\mathcal{O}(1/p
\alpha T)=\mathcal{O}(1/p\sqrt{n})$ for achieving $\mathcal{O}(1/T\alpha)$ optimization error. In Theorem \ref{sc-nonsmooth2}, if we assume $D_W\leq\mathcal{O}(\sqrt{n})$, $N=n^2$, we have $\varepsilon(\mathcal{A})=\mathcal{O}(1/pn)$. In this case, it is required $\mathcal{O}(n^2)$ inner iterations and $\mathcal{O}(1)$ outer iterations. The total iteration complexity is $\mathcal{O}(n^2)$.

Notice that this rate is not reducible for non-smooth problem \citep{bousquet2002stability}. In strongly convex non-smooth case, achieving uniform stability also requires $\mathcal{O}(n^2)$ iterations. Then, the iteration complexity of both of ME-$\mathcal{A}$ (for convex problem) and $\mathcal{A}$ (for strongly convex problem) are $\mathcal{O}(n^2)$. 

\paragraph{Smooth Setting.} Now we consider the performance of ME-$\mathcal{A}$ in smooth setting. In this case, applying first order algorithms to the inner problem have a linear convergence rate. The iteration complexity is $\mathcal{O}(n)$. It matches the rate of running SGD on smooth problem \cite{hardt2016train}.

\paragraph{No Computation Overhead.} Theoretically, the iteration complexity of  ME-$\mathcal{A}$ and $\mathcal{A}$ is the same. In practice, let $\mathcal{A}$ be the update rule of a first-order multi-epochs algorithm. Then, we can choose $T$ to be the number of epochs and $N$ to be the number of iterations in each epochs. ME-$\mathcal{A}$ has the same computational cost as $\mathcal{A}$.

\section{Other Discussion: Additional Related Work and Additional Experiments}
\paragraph{Additional Related Work.}In addition to gradient descent, some variants of SGD are also proven to be non-uniformly stable for convex non-smooth problems. We list the stability of most related results in Table \ref{table:nonsmooth}. We can see that there are also additional term in the bounds of the uniform stability of variants of SGD.

\paragraph{Adversarial Attack.} Adversarial examples were first introduced in \citep{szegedy2014intriguing}. Since then, adversarial attacks have received enormous attention \citep{papernot2016limitations,moosavi2016deepfool,carlini2017towards}. Nowadays, attack algorithms have become sophisticated and powerful. For example, Autoattack \citep{croce2020reliable} and Adaptive attack \citep{tramer2020adaptive}. Real-world attacks are not always norm-bounded \citep{kurakin2018adversarial}. \citet{xiao2022understanding} considered non-$\ell_p$ attacks .

\paragraph{Adversarially Robust Generalization.} Even enormous algorithms were proposed to improve the robustness of DNNs  \citep{madry2018towards,gowal2020uncovering,rebuffi2021fixing}, the performance was far from satisfactory. One major issue is the poor robust generalization, or robust overfitting \citep{rice2020overfitting}. 
A series of studies \citep{xiao2022adaptive, ozdaglar2022good} have delved into the concept of uniform stability within the context of adversarial training. However, these analyses focused on general Lipschitz functions, without specific consideration for neural networks.

\begin{table*}[htbp]
\caption{Uniform Stability for different variants of SGD algorithms in non-smooth convex minimization problem. Here $T$ is the number of iterations, $n$ is the number of samples, and $\alpha>0$ is the step size.}
    \centering
	    \begin{tabular}{cccc}
		    \toprule
		    &Algorithms&Upper Bounds& Lower Bounds\\
		    \midrule
		    \citep{bassily2020stability}& GD (full batch) & $\mathcal{O}(\sqrt{T}\alpha+T\alpha/n)$ & $\Omega(\sqrt{T}\alpha+T\alpha/n)$\\
		    \citep{bassily2020stability}& SGD (w/replacement) & $\mathcal{O}(\sqrt{T}\alpha+T\alpha/n)$ & $\Omega(\min\{1,T/n\}\sqrt{T}\alpha+T\alpha/n)$\\
		    \citep{bassily2020stability}& SGD (fixed permutation) & $\mathcal{O}(\sqrt{T}\alpha+T\alpha/n)$ & $\Omega(\min\{1,T/n\}\sqrt{T}\alpha+T\alpha/n)$\\
		    \citep{yang2021stability}& SGD (Pairwise Learning)  & $\mathcal{O}(\sqrt{T}\alpha+T\ln{T}\alpha/n)$ & / \\
		    \citep{wang2022stability}& Markov chain-SGD &$\mathcal{O}(\sqrt{T}\alpha+T\alpha/n)$&/ \\
		    \bottomrule
		    \end{tabular}
    \label{table:nonsmooth}
\end{table*}
\subsection{Additional Experiments}
\label{app:add}
In this section, we provide ablation study about the hperparameter of ME-$\mathcal{A}$. Specifically, the value of $p$ and the step size $\tau$ affect the performance of ME-$\mathcal{A}$. 

\paragraph{Affect of $p$.} We first consider the affect of the value of $p$. When $p=0$, ME-$\mathcal{A}$ reduces to SWA. These experiments also provide a comparison between ME-$\mathcal{A}$ and SWA. In the Theorem, $p$ does not affect the uniform stability as long as $\alpha\leq 1/p$. But $p$ is related to the optimal rate of iteration complexity. We switch $p$ from 0 to $1\times 10^{5}$ in the experiment of adversarial training with ME-$\mathcal{A}$ using TRADES loss on CIFAR-10. The results are provided in Table \ref{tab:p}. We can see that the value of $p$ does not give a major difference in the performance. The best performance is achieved when $p=1$.

\begin{table}[]
    \centering
    \caption{The affect of $p$ in ME-$\mathcal{A}$. We switch $p$ from 0 to $1\times 10^{5}$ in the experiment of adversarial training with ME-$\mathcal{A}$ using TRADES loss.}
    \begin{tabular}{cc}
    \toprule
      $p=$   & Robust Accuracy \\
         \midrule
       0 (Reduced to SWA)   & 63.17 \\
       $1\times 10^{-5}$& 63.38 \\
              $1\times 10^{-4}$ &63.28 \\
    $1\times 10^{-3}$& 63.19 \\
    $1\times 10^{-2}$& 62.79 \\
    $1\times 10^{-1}$& 62.69 \\
    $1\times 10^{0}$ &63.86 \\
    $1\times 10^{1}$ &62.89 \\
    $1\times 10^{2}$ &63.47 \\
    $1\times 10^{3}$ &62.40 \\
    $1\times 10^{4}$ &62.79 \\
    $1\times 10^{5}$& 63.41 \\
    \bottomrule
    \end{tabular}
    \label{tab:p}
\end{table}

\paragraph{Affect of $\tau$.} Secondly, we consider the affect of the value of $\tau_t$. When $\tau_t=0$, ME-$\mathcal{A}$ reduces to proximal point methods (PPM). We simply let $\tau_t=\tau$, \emph{i.e.,} fixed stepsize in the experiemnts. These experiments also provide a comparison between ME-$\mathcal{A}$ and PPM. In the Theorem, $\tau$ should be large and closed to 1, especial for large $t$. We switch $\tau$ from 0 to $0.995$ in the experiment of adversarial training with ME-$\mathcal{A}$ using TRADES loss on CIFAR-10. The results are provided in Table \ref{tab:tau}. We can see that the robust accuracy decreases as $\tau$ decreases. The best performance is achieved when $\tau=0.995$. Therefore, $\tau$ is better to be large. In the experiments in the main contents, $\tau_t$ increases as $t$.

\begin{table}[]
    \centering
    \caption{The affect of $\tau$ in ME-$\mathcal{A}$. We switch $\tau$ from 0 to $0.995$ in the experiment of adversarial training with ME-$\mathcal{A}$ using TRADES loss.}
    \begin{tabular}{cc}
    \toprule
      $\tau=$   & Robust Accuracy \\
         \midrule
       0 (Reduced to PPM)   & 56.86 \\
        $0.1$   & 57.02 \\
       $0.3$& 57.78 \\
              $0.5$ &58.34 \\
    $0.7$& 60.32 \\
    $0.9$& 62.57 \\
    $0.995$& 63.17 \\
    \bottomrule
    \end{tabular}
    \label{tab:tau}
\end{table}
\end{document}